\documentclass[accepted]{article}

\usepackage{microtype}
\usepackage{graphicx}
\usepackage{subfigure}
\usepackage{booktabs} % for professional tables

\usepackage{hyperref}

\usepackage[accepted]{icml2025}

\usepackage{hyperref}
\usepackage{url}

\usepackage[utf8]{inputenc} % allow utf-8 input
\usepackage[T1]{fontenc}    % use 8-bit T1 fonts
\usepackage{url}            % simple URL typesetting
\usepackage{amsfonts}       % blackboard math symbols
\usepackage{nicefrac}       % compact symbols for 1/2, etc.
\usepackage{microtype}      % microtypography
 \usepackage{array,multirow,graphicx}

\usepackage[printonlyused]{acronym}
\usepackage[english]{babel}
\usepackage{amsmath, amssymb}
\usepackage{amsthm}
\usepackage{mathrsfs}
\usepackage{enumerate}
\usepackage{booktabs}
\usepackage{cleveref}
\usepackage{subcaption}
\usepackage{xurl}

\usepackage{color}
\definecolor{codegreen}{rgb}{0,0.6,0}
\definecolor{codegray}{rgb}{0.5,0.5,0.5}
\definecolor{codepurple}{rgb}{0.58,0,0.82}
\definecolor{backcolour}{rgb}{0.95,0.95,0.92}

\definecolor{brickred}{rgb}{0.8, 0.25, 0.33}

\newcommand{\R}{\mathbb{R}}

\newcommand{\N}{\mathcal{N}_0}

\newcommand{\X}{\boldsymbol{X}}

\newcommand{\Y}{\boldsymbol{Y}}
\newcommand{\x}{{\boldsymbol{x}}}
\newcommand{\w}{\boldsymbol{w}}
\newcommand{\U}{\boldsymbol{u}}
\newcommand{\V}{\boldsymbol{v}}
\newcommand{\y}{{\boldsymbol{y}}}

\newcommand{\zero}{{\boldsymbol{0}}}
\newcommand{\z}{{\boldsymbol{z}}}

\newcommand{\etab}{{\boldsymbol{\eta}}}
\newcommand{\bmu}{{\boldsymbol{\mu}}}
\newcommand{\bxi}{{\boldsymbol{\xi}}}
\newcommand{\ygenSDE}{{\tilde{\boldsymbol{y}}}}
\newcommand{\ygenODE}{{\widehat{\boldsymbol{y}}}}
\newcommand{\pdata}{p_\mathrm{data}}
\newcommand{\pgenSDE}{{\tilde{p}}}
\newcommand{\pgenODE}{{\widehat{p}}}
\newcommand{\qgenSDE}{{\tilde{q}}}
\newcommand{\qgenODE}{{\widehat{q}}}
\newcommand{\bSigma}{{\boldsymbol{\Sigma}}}
\newcommand{\SigmagenSDE}{{\tilde{\bSigma}}}
\newcommand{\SigmagenODE}{{\widehat{\bSigma}}}

\newcommand{\fSDE}{{\tilde{f}}}
\newcommand{\fODE}{{\widehat{f}}}

\newcommand{\ySDEEM}{{\tilde{\boldsymbol{y}}^{\Delta,\text{EM}}}}
\newcommand{\ySDEEI}{{\tilde{\boldsymbol{y}}^{\Delta,\text{EI}}}}
\newcommand{\ySDEDDPM}{{\tilde{\boldsymbol{y}}^{\Delta,\text{DDPM}}}}
\newcommand{\yODEEuler}{{\widehat{\boldsymbol{y}}}^{\Delta,\text{Euler}}}
\newcommand{\yODEHeun}{{\widehat{\boldsymbol{y}}}^{\Delta,\text{Heun}}}
\newcommand{\yODERK}{{\widehat{\boldsymbol{y}}}^{\Delta,\text{RK$4$}}}

\newcommand{\pdataemp}{p_\mathrm{data}^{\text{\tiny emp.} }}
\newcommand{\pemp}{p^{\text{\tiny emp.} }}

\newcommand{\A}{\boldsymbol{A}}
\newcommand{\G}{\boldsymbol{\Gamma}}
\newcommand{\I}{\boldsymbol{I}}
\newcommand{\eps}{\varepsilon}
\newcommand{\E}{\mathbb{E}}
\DeclareMathOperator{\ADSN}{ADSN}

\DeclareMathOperator{\Wass}{\text{$\textbf{W}_{2}$}}
\DeclareMathOperator{\Wemp}{\text{$\textbf{W}^{\text{\tiny emp.} }_{2}$}}

\DeclareMathOperator{\Cov}{Cov}
\newcommand{\bt}{\mathbf{t}}
\newcommand{\OMN}{{M\times N}}
\newtheorem{prop}{Proposition}
\newtheorem{assumption}{Assumption}

\icmltitlerunning{Diffusion models for Gaussian distributions: Exact solutions and Wasserstein errors}

\begin{document}

\twocolumn[
\icmltitle{\large Diffusion models for Gaussian distributions: Exact solutions and Wasserstein errors}



\icmlsetsymbol{equal}{*}

\begin{icmlauthorlist}
\icmlauthor{Emile Pierret}{univorleans}
\icmlauthor{Bruno Galerne}{univorleans,IUF}
\end{icmlauthorlist}

\icmlaffiliation{univorleans}{Université d'Orléans, Université de Tours, CNRS, IDP, UMR 7013, Orléans, France}
\icmlaffiliation{IUF}{Institut universitaire de France (IUF)}

\icmlcorrespondingauthor{Emile Pierret}{emile.pierret@univ-orleans.fr}
\icmlcorrespondingauthor{Bruno Galerne}{bruno.galerne@univ-orleans.fr}

\icmlkeywords{Diffusion models, image generation, differential equations, discretization schemes}

\vskip 0.3in
]



\printAffiliationsAndNotice{}  % leave blank if no need to mention equal contribution

\begin{abstract}
Diffusion or score-based models recently showed high performance in image generation.
They rely on a forward and a backward stochastic differential equations (SDE). The sampling of a data distribution is achieved by numerically solving the backward SDE or its associated flow ODE.
Studying the convergence of these models necessitates to control four different types of error: the initialization error, the truncation error, the discretization error and the score approximation.
In this paper, we theoretically study the behavior of diffusion models and their numerical implementation when the data distribution is Gaussian.
Our first contribution is to derive the analytical solutions of the backward SDE and the probability flow ODE and to prove that these solutions and their discretizations are all Gaussian processes.
Our second contribution is to compute the exact Wasserstein errors between the target and the numerically sampled distributions for any numerical scheme.
This allows us to monitor convergence directly in the data space, while experimental works limit their empirical analysis to Inception features.
An implementation of our code is available \href{https://github.com/emilePi/Diffusion-models-for-Gaussian-distributions-Exact-solutions-and-Wasserstein-errors}{online}.
% Our experiments show that the recommended numerical schemes from the diffusion models literature are also the best sampling schemes for Gaussian distributions.
\end{abstract}

\section{Introduction}

Over the last five years, diffusion models have proven to be a highly efficient and reliable framework for generative modeling \citep{song_generative_2019-1, ho_ddpm_2020_neurips, Song_etal_denoising_diffusion_implicit_models_ICLR2021, yang_song_score_based_sde_2021_ICLR, dhariwal_diffusion_beats_gan_2021_neurips, Karras_Elucidating_desing_diffusion_neurips_2022}.
First introduced as a discrete process, Denoising Diffusion Probabilistic Models (DDPM) \citep{ho_ddpm_2020_neurips} can be studied as a reversal of a continuous Stochastic Differential Equation (SDE) \citep{yang_song_score_based_sde_2021_ICLR}.
A forward SDE progressively transforms the initial data distribution by adding more and more noise as time progresses.
Then, the reversal of this process, called backward SDE, allows us to approximately sample the data distribution starting from Gaussian white noise.
Moreover, the SDE is associated with an Ordinary Differential Equation (ODE) called the probability flow~\citep{yang_song_score_based_sde_2021_ICLR}.
This flow preserves the same marginal distributions as the backward SDE and provides another way to sample the score-based generative model.

An important issue about diffusion models is the theoretical guarantees of convergence of the model:
How close to the data distribution is the generated distribution?
%{
There are four main sources of errors to study to derive theoretical guarantees for diffusion models: (a) the \textit{initialization error} is induced when approximating the marginal distribution at the end of the forward process by a standard Gaussian distribution. (b) The \textit{discretization error} comes from the resolution of the SDE or the ODE by a numerical method. (c) The \textit{truncation error} occurs because the backward time integration is stopped at a small time $\eps>0$ to avoid numerical instabilities due to ill-defined score function near the origin. (d) The \textit{score approximation error} accounts for the mismatch between the ideal score function and the one given by the network trained using denoising score-matching.

Despite these numerous sources of errors,
a lot of numerical and theoretical research has been conducted to assess the generative capacity of diffusion models.
Several articles \citep{Giulio_how_much_enough_Entropy_2023,Karras_Elucidating_desing_diffusion_neurips_2022}
provide strong experimental studies for the choices of sampling parameters.
On the theoretical side, several works derive upper bounds on the 1-Wasserstein or TV distance between the data and the model distributions by making assumptions on the $L^2$-error between the ideal and learned score functions and on the compacity of the support of the data \citep{chen_flow_ode_provably_fast_2023_ICLR,Lee_convergence_score_based_neurips_2022,de_bortoli_et_al_diffusion_schrodinger_bridge_2021_neurips,chen_sampling_as_easy_as_learning_Score_ICLR_2023,convergence_score_based_general_data_distribution_lee_neurips_2022,benton_nearly_d_linear_convergence_bounds_ICLR_2024}, eventually under an additional manifold assumption \citep{de_bortoli_convergence_diffusion_manifold_hypothesis_2022_TMLR,wenliang_score_based_learn_manifold_neurips_2022,Chen_Score_approx_diff_model_low_dim_ICLR_2023}.
Yet, on one hand, to the best of our knowledge, the derived theoretical bounds mostly rely on worst case scenario and are not tight enough to explain the practical efficiency of diffusion models.
On the other hand, numerical considerations mostly rely on Inception feature distributions through the FID metric~\citep{Heusel_etal_GANs_local_nash_equilibrium_NIPS2017}.

Ideally, given a data distribution of interest, one would like to have an adapted estimation of the discrepancy between the data and the diffusion model samples, thus enabling adaptive hyperparameter selection for the sampling procedure.
As a first step towards reaching this goal, in the present work we study diffusion models applied to Gaussian data distributions.
While this setting has a priori no practical interest, since simulating Gaussian variates does not require a diffusion model, it provides a large parametric family of distributions for which the errors involved in diffusion model sampling can be completely understood.

When restricting the data distribution to be Gaussian, the resulting score function is a simple linear operator. Exploiting this specificity allows us to make the following contributions \textbf{under the assumption that the data distribution is Gaussian}:
\begin{itemize}
  \item We give the exact solutions for both the backward SDE and the probability flow ODE.
  \item We fully describe the Gaussian processes that occur when using classical sampling discretization schemes.
  \item We derive exact 2-Wasserstein errors for the corresponding sample distributions and are able to assess the influence of each error type on these errors, as illustrated by Figure~\ref{fig:cifar_measures}.
\end{itemize}
Our theoretical study allows for an analytical evaluation of any numerical sampler, either stochastic or deterministic.
In particular, it confirms the strength of best practice scheme such as Heun's method for the ODE flow~\citep{Karras_Elucidating_desing_diffusion_neurips_2022}. 
%{
Our source code is available \href{https://github.com/emilePi/Diffusion-models-for-Gaussian-distributions-Exact-solutions-and-Wasserstein-errors}{online} and can be applied to any Gaussian data distribution of interest and gives insight to calibrate parameters of a diffusion sampling algorithm, e.g. by straightforwardly generalizing our study to higher order linear numerical schemes.

While our theoretical analysis relies on an exactly known score function,
we conduct additional experiments to assess the impact of the score approximation error.
Surprisingly, in the context of texture synthesis, we show that with a score neural network trained for modeling a specific Gaussian microtexture, a stochastic Euler-Maruyama sampler is more faithful to the data distribution than Heun's method, thus highlighting the importance of the score approximation error in practical situations.

\textbf{Plan of the paper:} First, we recall in \Cref{sec:reminders} the continuous framework for SDE-based diffusion models.
\Cref{sec:exact_sde_ode_solutions} presents our main theoretical results detailing the exact backward SDE and probability flow ODE solutions when supposing the data distribution to be Gaussian.
\Cref{sec:wasserstein_errors} gives explicit Wasserstein error formulas when sampling the corresponding processes,
yielding an ablation study for comparing the influence of each error type on several sampling schemes.
In \Cref{sec:score_approximation_adsn}, we study numerically a special case of Gaussian distribution for texture synthesis in order to evaluate the influence of the score approximation error occurring with a standard network architecture. Finally, we address discussion and limitations of our framework in \Cref{sec:discussion_limitations}.

\newlength{\cifarwidth}
\setlength{\cifarwidth}{\columnwidth}

\begin{figure*}
          \centering
          \begin{tabular}{cc}
          \includegraphics[width=\cifarwidth]{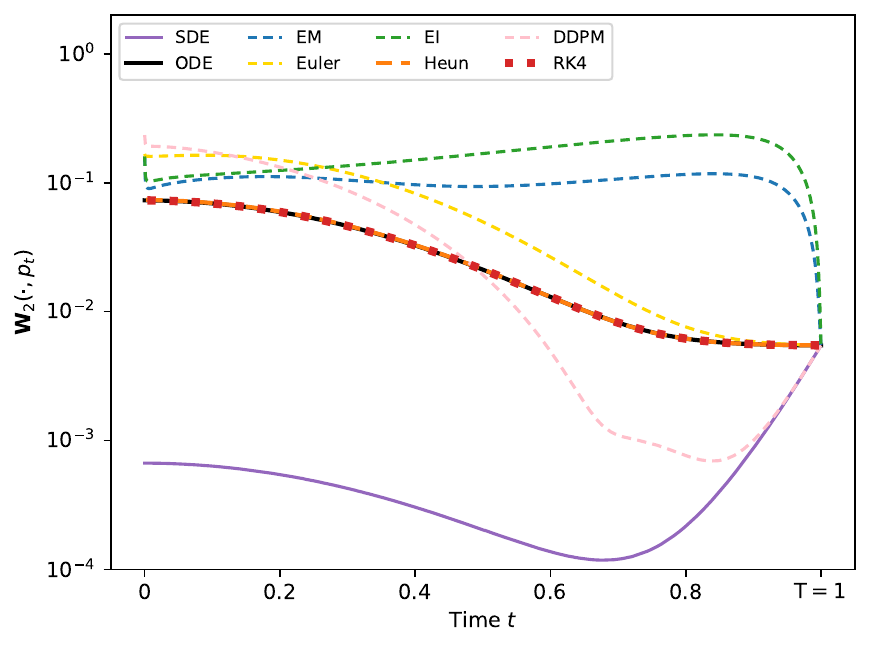}
          &
          \includegraphics[width = \cifarwidth]{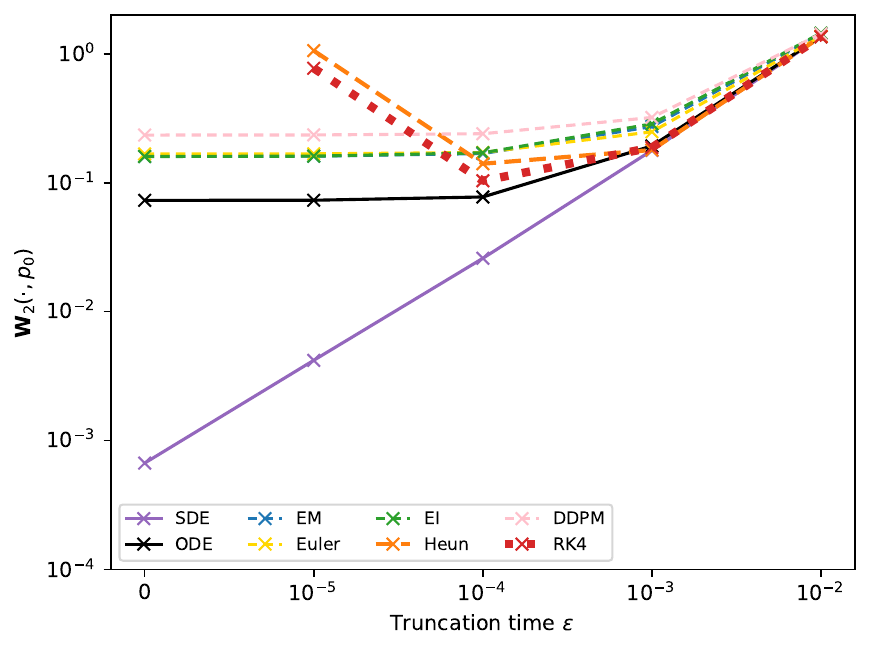}\\
         (a) Initialization error along the integration time  &
          (b) Truncation error for different truncation time $\varepsilon$
\end{tabular}
 \caption{\small \label{fig:cifar_measures} \textbf{Wasserstein errors for the diffusion models associated with the CIFAR-10 Gaussian.} Left: Evolution of the Wasserstein distance between $p_t$ and the distributions associated with the continuous SDE, the continuous flow ODE and four discrete sampling schemes with standard $\N$ initialization, either stochastic (Euler-Maruyama (EM), Exponential Integrator (EI) and Diffusion Denoising Probabilistic Models (DDPM)) or deterministic (Euler, Heun and Runge-Kutta 4 (RK4)).
While the continuous SDE is less sensitive than the continuous ODE (as proved by Proposition~\ref{prop:marginal_generatives}), the initialization error impacts all discrete schemes with a comparable order of magnitude.
Heun and RK4 methods have the lowest error and are very close to the theoretical ODE, except for the last step (which is not represented because unstable) that is usually discarded when using time truncation.
Right: Wasserstein errors due to time truncation for various truncation times $\eps$.
Using time truncation increases the error for all the methods except Heun's scheme and RK4 due to instability near the origin.
Interestingly, for the standard practice truncation time $\varepsilon = 10^{-3}$, all numerical schemes have a comparable error close to their continuous counterparts.}
\end{figure*}

\section{Preliminaries: Score-Based Models through Diffusion SDEs}
\label{sec:reminders}

This preliminary section follows the seminal work of Song \emph{et al.} \citep{yang_song_score_based_sde_2021_ICLR} and introduces specific notation to differentiate the exact backward process and the generative backward process obtained when starting from a white noise.
Given a target distribution $\pdata$ over $\R^d$, the forward diffusion process is the following variance preserving SDE
 \begin{equation}
 \label{eq:sde_forward}
 	d\x_t  =-\beta_t \x_t dt + \sqrt{2\beta_t}d\w_t,  0 \leq t \leq T,   \x_0 \sim \pdata
 \end{equation}
 where $\left(\w_t\right)_{t \geq 0}$ is a $d$-dimensional Brownian motion and $\beta$ is a positive weight function.
 The distribution $\pdata$ is noised progressively and the function $\beta$ is the variance of the added noise by time unit.
 We denote by $p_t$ the density of $\left(\x_t\right)$ for $t > 0$ since $\pdata$ can be supported on a lower-dimensional manifold~\citep{de_bortoli_convergence_diffusion_manifold_hypothesis_2022_TMLR}.
The SDE is designed so that $p_T$ is close to the Gaussian standard distribution that we denote
 $\N$ in whole paper.
Under some assumptions on the distribution $\pdata$ \citep{pardoux_1986}, the backward process $\left(\x_{T-t}\right)_{0\leq t \leq T}$ verifies the backward SDE
 \begin{align}
 \label{eq:sde_backward}
 d\y_t = \beta_{T-t}[\y_t + 2  \nabla \log & p_{T-t}(\y_t)]dt + \sqrt{2\beta_{T-t}}d\w_t,\notag \\
 &  0 \leq t < T, \quad \y_0 \sim p_T. 
 \end{align}
The objective is now to solve this reverse equation  in order to sample $\y_T \sim p_{\mathrm{data}}$.
However, the distribution $p_T$ is in general not known, and image\footnote{Although we may refer to data as images, our analysis is fully general and applies to any vector-valued diffusion model.} generation is achieved by sampling
 \begin{align}
 \label{eq:sde_backward_generative}
 d\ygenSDE_t = \beta_{T-t}[\ygenSDE_t + 2  \nabla \log & p_{T-t} (\ygenSDE_t) ]dt  + \sqrt{2\beta_{T-t}}d\w_t  , \notag \\
  & 0 \leq t < T, \quad \ygenSDE_0 \sim \N  .
 \end{align}
Note that approximating $p_T$ by $\N$ for the initialization $\y_0$ implies that the solution of the SDE of \Cref{eq:sde_backward_generative} does not exactly follow the target distribution $\pdata$ at final time.
An alternative way to approximately sample $\pdata$ is to use that every diffusion process is associated with a deterministic process whose trajectories share the same marginal probability densities $\left(p_t\right)_{0< t \leq T}$ as the SDE \citep{yang_song_score_based_sde_2021_ICLR}.
The deterministic process associated with \Cref{eq:sde_backward} is
     \begin{align}
 \label{eq:flow_ode}
 	d\x_t = -\beta_t[\x_t + \nabla\log & p_t(\x_t)]dt, \notag \\
    & 0 < t \leq T,  \x_0 \sim p_{\mathrm{data}}.
  \end{align}
This ODE can be solved in reverse-time to sample $\x_0$ from $\x_T \sim p_T$. Given $(\x_t)_{0\leq t \leq T}$ solution of \Cref{eq:flow_ode}, $(\x_{T-t})_{0\leq t \leq T}$ is solution of
	 \begin{equation}
 \label{eq:flow_reverse_ode}
 	d\y_t = \beta_{T-t}\left[\y_t + \nabla\log p_{T-t}(\y_t)\right]dt, \quad 0 \leq t < T.
  \end{equation}
	Again, in practice, the ODE which is considered to achieve image generation is
 \begin{align}
 \label{eq:flow_ode_gen}
 	d\ygenODE_t = \beta_{T-t}[\ygenODE_t +\nabla\log & p_{T-t}(\ygenODE_t)]dt, \notag \\
    &  0 \leq t < T, \ygenODE_0 \sim \N,
\end{align}
  where $p_T$ is replaced by $\N$.
  As a consequence of this approximation, the property of conservation of the marginals $\left(p_t\right)_{0\leq t \leq T}$ does not occur.
  We denote by $\left(\qgenSDE_t\right)_{0\leq t \leq T}$, respectively $\left(\qgenODE_t\right)_{0\leq t \leq T}$, the marginals of $\left(\ygenSDE_t\right)_{0\leq t \leq T}$ and $\left(\ygenODE_t\right)_{0\leq t \leq T}$ and $\pgenSDE_t = \qgenSDE_{T-t}$, $\pgenODE_t = \qgenODE_{T-t}$ the marginals of $\left(\ygenSDE_{T-t}\right)_{0\leq t \leq T}$ and $\left(\ygenODE_{T-t}\right)_{0\leq t \leq T}$ such that $\pgenSDE_t$ and $\pgenODE_t$ are approximations of $p_t$.

\section{Exact SDE and ODE Solutions}
\label{sec:exact_sde_ode_solutions}

 Our approach relies on deriving explicit solutions to the various SDE and ODE.
 We begin with the forward SDE in full generality obtained by applying the variation of constants (see the proof in Appendix~\ref{appendix:proof_prop:solution_forward_sde}). This resolution also provides an ODE verified by the covariance matrix of $\x_t$, that we denote $\bSigma_t = \Cov(\x_t)$.

 \begin{prop}[Solution of the forward SDE]
 \label{prop:solution_forward_sde}
  	The strong solution of \Cref{eq:sde_forward} can be written as:
 	\begin{equation}
 	\label{eq:solution_forward_sde}
 		\x_t = e^{-B_{t}}\x_0 + \etab_t, \quad 0 \leq t \leq T,
 	\end{equation}
 	where $B_{t} = \int_0^t \beta_{s}ds$ and $\etab_t = e^{-B_{t}}\int_0^te^{B_{s}} \sqrt{2\beta_{s}}d\w_s$ is a Gaussian process independent of $\x_0$ whose covariance matrix is $(1-e^{-2B_{t}})\I$. Consequently, the covariance matrix $\bSigma_t$ of $\x_t$ is
 	\begin{equation}
 	\label{eq:covariance_Expression_forward}
 		\bSigma_t = e^{-2B_{t}}\bSigma + (1-e^{-2B_{t}})\I.
 	\end{equation}
 	where $\bSigma$ is the covariance matrix of $\x_0 \sim \pdata$.
	Futhermore, $\bSigma_t$ is invertible for $t > 0$ and verifies the matrix-valued ODE
 	\begin{equation}
 	\label{eq:ode_covariance}
 		d\bSigma_t = 2\beta_t(\I-\bSigma_t)dt
 		, \quad 0 < t \leq T.
 	\end{equation}
 \end{prop}

For a general data distribution $\pdata$, solving the backward SDE is infeasible, the main reason being that the expression of the score function to integrate is unknown.
To circumvent this obstacle, we now suppose that the data distribution is Gaussian.

\begin{assumption}[Gaussian assumption]
\label{assumption:Gaussian}
	$\pdata$ \textbf{is a centered Gaussian distribution $\mathcal{N}(\zero,\bSigma)$.}
\end{assumption}

Note that $\bSigma$ may be non-invertible and thus $\pdata$ supported on a strict subspace of $\R^d$, a special case of manifold hypothesis.
Consequently, the matrix $\bSigma_{t}$ is in general only invertible for $t > 0$.
Under Gaussian assumption, $(\x_t)$ is a Gaussian process with marginal distributions $p_t = \mathcal{N}(\zero,\bSigma_t)$ and consequently the score is the linear function
\begin{equation}
	\label{eq:score_expression_Gaussian}
	\nabla \log p_t(\x) = -\bSigma_t^{-1}\x,\quad 0 < t \leq T.
\end{equation}
Note that the linearity of the diffusion score characterizes Gaussian distributions as detailed by Proposition~\ref{prop:Gaussian_linear_score} in Appendix~\ref{appendix:Gaussian_linear_score}.

The cornerstone of our work is that under Gaussian assumption we can derive an exact solution of the backward SDE, without supposing that the initial condition is Gaussian.

 \begin{prop}[Solution of the backward SDE under Gaussian assumption]
 \label{prop:solution_backward_sde_gaussian_assumption}
 	 For $\pdata = \mathcal{N}(\zero,\bSigma)$,
  the strong solution to the SDE of \Cref{eq:sde_backward}
  \begin{equation}
       d\y_t = \beta_{T-t}[\y_t + 2\nabla_\y \log p_{T-t}(\y_t)]dt + \sqrt{2\beta_{T-t}}d\w_t
  \end{equation}
  with $\y_0$ following any initial distribution can be written as
 	\begin{equation}
 	\label{eq:solution_backward_sde}
 	\y_t = e^{-(B_{T}-B_{T-t})}\bSigma_{T-t}\bSigma^{-1}_T\y_0 +\bxi_t, \quad 0 \leq t \leq T
 	\end{equation}
 	where $\bxi_t$ is a Gaussian process with covariance matrix \begin{equation}
 	    \Cov(\bxi_t) = \bSigma_{T-t}  - e^{-2(B_{T}-B_{T-t})}\bSigma^2_{T-t}\bSigma^{-1}_{T}.
 	\end{equation} Finally, if $\Cov(\y_0)$ and $\bSigma$ commute,
 	% \begin{equation}
 	% \label{eq:covariance_backward_sde_general_not_commuting}
 	% 	\Cov(\y_t) = \bSigma_{T-t} + e^{-2(B_{T}-B_{T-t})}\bSigma_{T-t}^2\bSigma^{-1}_{T}\left(\bSigma^{-1}_{T-t}\Cov(\y_0)\bSigma^{-1}_{T}\bSigma_{T-t}-\I \right),
 	% \end{equation}
 	% and in particular, if $\Cov(\y_0)$ and $\bSigma$ commute,
 	\begin{align}
 	\label{eq:covariance_backward_sde_general}
 		& \Cov(\y_t)  = \bSigma_{T-t}\notag
        \\
        & + e^{-2(B_{T}-B_{T-t})}\bSigma_{T-t}^2 \bSigma_T^{-2}\left[\Cov(\y_0)-\bSigma_T\right].
 	\end{align}
 \end{prop}

 While not as straightforward as the forward case, the proof also relies on applying the variation of constants and is given in Appendix~\ref{appendix:proof_prop:solution_backward_sde}.
% {
 Note that if $\y_0$ is correctly initialized at $p_T$, $y_{T-t} \sim p_t$ at each time $0\leq t \leq T$. 
As shown by the following proposition (proved in Appendix~\ref{appendix:proof_prop:solution_ode_Gaussian}),
the flow ODE  also has an explicit solution under Gaussian assumption which is related to optimal transport (OT).

 \begin{prop}[Solution of the ODE probability flow under Gaussian assumption]
  \label{prop:solution_ode_Gaussian}
	%The solution to the probability flow ODE~\eqref{eq:flow_ode} under Gaussian assumption corresponds to the optimal transport map between $p_T$ and $\pdata$. More precisely, for any $\y_0$,
 %{ 
 For $\pdata = \mathcal{N}(\zero,\bSigma)$, the solution to the reverse-time probability flow ODE of \Cref{eq:flow_reverse_ode}
 	 \begin{equation}
% \label{eq:flow_reverse_ode}
 	d\y_t = \left[\beta_{T-t}\y_t +\beta_{T-t} \nabla_{\y}\log p_{T-t}(\y_t)\right]dt, \quad 0 \leq t < T
  \end{equation}
  for any $\boldsymbol{y}_0$ is
	\begin{equation}
	\y_t = \bSigma^{-1/2}_T\bSigma^{1/2}_{T-t}\y_0, \quad 0 \leq t \leq T,
    \end{equation}
	%{
  which is the application of the OT map between $p_T$ and $p_{T-t}$ to the initial condition $\boldsymbol{y}_0$. Consequently, if $\Cov(\y_0)$ and $\bSigma$ commute,
%the covariance matrix $\Cov(\y_t)$ verifies
% 	\begin{equation}
%  	\label{eq:covariance_backward_ode_general_not_commuting}
%  		\Cov(\y_t) = \bSigma^{-1/2}_{T}\bSigma^{1/2}_{T-t}\Cov(\y_0)\bSigma^{1/2}_{T-t}\bSigma^{-1/2}_{T}, \quad 0 \leq t \leq T,
%  	\end{equation}
% and in particular, if $\Cov(\y_0)$ and $\bSigma$ commute,
		\begin{equation}
 	\label{eq:covariance_backward_ode_general}
 		\Cov(\y_t) = \bSigma_{T-t}\bSigma^{-1}_{T}\Cov(\y_0), \quad 0 \leq t \leq T.
 	\end{equation}
 \end{prop}

%\begin{rem}
\paragraph{Remark}
Here we must highlight a subtle issue: Whatever the initial distribution of $\y_0$ is, the ODE solution consists in applying the OT map between %{
$p_T$ and $p_{T-t}$ at time $t$. 
If $\y_0$ follows $p_T$, then $\y_{T-t} \sim p_t$ at each time $0\leq t \leq T$. 
But since in practice one cannot truly sample $p_T$ and uses $\y_0 \sim \N$ instead,
the resulting flow is not an OT flow (even though it involves an OT mapping) and the distribution of $\y_T$ differs from $\pdata$.
%\end{rem}

For the sake of completeness, we extend Proposition~\ref{prop:solution_backward_sde_gaussian_assumption} and Proposition~\ref{prop:solution_ode_Gaussian} to the non centered case $\pdata=\mathcal{N}(\bmu,\bSigma)$ in Appendix~\ref{appendix:solutions_with_non_zero_mean}.

\paragraph{Links with related work.}
Some parts of the previous propositions have been stated in previous work.

\Cref{eq:solution_forward_sde} of Proposition~\ref{prop:solution_forward_sde} is given without proof in \citep{Convergence_Analysis_general_diffusion_models_Gao_2024_arxiv}, the variance ODE, 
that we generalize here to the full covariance matrix (\Cref{eq:ode_covariance}), is given in \citep{yang_song_score_based_sde_2021_ICLR}, \citep[Equation 6.20]{sarkka_applied_SDE_2019}), and the score expression under Gaussian assumption is reported in several recent references \cite{albergo_stochastic_interpolants_2023_arxiv,
Zach_Gaussian_miwture_diffusion_models_2024, 
Zach_explicit_diffusion_Gaussian_mixture, 
shah_learning_mixture_Gaussians_Neurips_2023}.

To the best of our knowledge Proposition 2 is new and is the cornerstone for our analytical and numerical study.
%Gaussian mixtures have been studied in the context of diffusion models \citep{Zach_Gaussian_miwture_diffusion_models_2024, Zach_explicit_diffusion_Gaussian_mixture, shah_learning_mixture_Gaussians_Neurips_2023} since they also provide an explicit analytical score.
%However, solving exactly the backward SDE is not feasible for Gaussian mixtures as far as we know.
%{The solution of the ODE under Gaussian assumption is given in \cite{Binxu_wang_hidden_linear_Structure_Score_based_workshop_neurips}}. 
%The ODE \cite{Fokker_Planck_not_0T_Lavenant_Applied_Mathematics_Letters_2022,understanding_DDPM_through_OT_Khrulkov_ICLR_2023} can be interpreted in the infinite time as an optimal transport (OT) between the prior distribution and the Gaussian standard distribution { under some assumptions, including Gaussian assumption.}

%{
The relation between OT and probability flow ODE has been discussed in \citep{Fokker_Planck_not_0T_Lavenant_Applied_Mathematics_Letters_2022,understanding_DDPM_through_OT_Khrulkov_ICLR_2023}. \cite{Fokker_Planck_not_0T_Lavenant_Applied_Mathematics_Letters_2022} show that, in general, the flow ODE solution is not an OT between $\pdata$ and $\N$ at infinite time $T \rightarrow +\infty$, thus contradicting a conjecture of \cite{understanding_DDPM_through_OT_Khrulkov_ICLR_2023}.
Yet, they briefly discuss the Gaussian case as special case for which the conjecture is valid. Indeed, \cite{understanding_DDPM_through_OT_Khrulkov_ICLR_2023} derive the solution of the flow ODE under Gaussian assumption at infinite time horizon (Appendix B \cite{understanding_DDPM_through_OT_Khrulkov_ICLR_2023}).
More recently, an expression of the solution of the flow ODE relying on the eigendecomposition of the covariance matrix of the data in Gaussian case is given in \citep{Binxu_wang_hidden_linear_Structure_Score_based_workshop_neurips} assuming $\y_0\sim\N$.
None of these works discuss the mismatch between the OT map and the initialization of $\y_0$.
Our Proposition~\ref{prop:solution_ode_Gaussian} highlights that the generated process is not an OT flow due to the initialization error. %While in general the solution is not an OT transport, the Gaussian case is specific and it has been shown thet the solution corresponds to a transport map between $\N$ and $\pdata$ in this infinite time setting.
%Since in practice $\boldsymbol{y}_0$ follows the standard distribution $\N$ and not $p_T$, there is no optimal transport between $\N$ and its image by the ($p_T$ to $p_{\mathrm{data}}$) OT mapping. This observation is new to the best of our knowledge. Our proof is also direct and new to the best of our knowledge. 
% Our main point is that by using the backward probability flow starting from $\N$, one applies the ($p_T$ to $p_{\mathrm{data}}$) OT mapping and so there is a mismatch between the source distribution and the OT mapping. This is why we can assert that this is not an OT (even though it involves an OT mapping) at finite time $T$.

\section{Exact Wasserstein Errors}
\label{sec:wasserstein_errors}

The specificity of the Gaussian case allows us to study precisely the different types of error with the expression of the explicit solution of the backward SDE. In what follows, we designate by Wasserstein distance the 2-Wasserstein distance which is known in closed forms when applied to Gaussian distributions \citep{Dowson_Landau_Frechet_distance_between_multivariate_normal_distributions_1982}.
For two centered Gaussians $\mathcal{N}(\zero,\bSigma_1)$ and $\mathcal{N}(\zero,\bSigma_2)$ such that the matrices $\bSigma_1$ and $\bSigma_2$ are simultaneously diagonalizable with respective eigenvalues $\left(\lambda_{i,1}\right)_{1\leq i \leq d},\left(\lambda_{i,2}\right)_{1\leq i \leq d},$ 
\begin{equation}
\label{eq:W2_Gaussian_commute_diago}
	\Wass(\mathcal{N}(\zero,\bSigma_1),\mathcal{N}(\zero,\bSigma_2))^2 =  \sum_{1 \leq i \leq d} (\sqrt{\lambda_{i,1}}-\sqrt{\lambda_{i,2}})^2
\end{equation}
as used in \citep{Ferradans_Static_dynamic_texture_0T_2013}. In the literature, the quality of the diffusion models is measured with FID~\citep{Heusel_etal_GANs_local_nash_equilibrium_NIPS2017}
which is the $\Wass$-error between Gaussians fitted to the Inception features~\citep{Szegedy_etal_Rethinking_the_Inception_Architecture_for_Computer_Vision_CVPR2016} of two discrete datasets.
Here we use the $\Wass$-errors directly in data space, which is more informative and allows us to provide theoretical $\Wass$-errors.
To illustrate our theoretical results, we consider the CIFAR-10 Gaussian distribution, that is, the Gaussian distribution such that $\bSigma$ is the empirical covariance of the CIFAR-10 dataset.
As shown in Appendix~\ref{appendix:Gaussian_cifar10_samples}, images produced by this model are not interesting due to a lack of structure, but the corresponding covariance has the advantage of reflecting the complexity of real data.

 \paragraph{The initialization error.}
As discussed in Sections~\ref{sec:reminders} and~\ref{sec:exact_sde_ode_solutions},
the marginals of
both generative processes  $\ygenSDE$ and $\ygenODE$ following respectively \Cref{eq:flow_ode_gen} and \Cref{eq:sde_backward_generative}
slightly differs from $p_t$ due to their common white noise initial condition.
This implies an error that we call the initialization error. The distance between $\left(\pgenSDE_t\right)_{0 \leq t \leq T}$, $\left(\pgenODE_t\right)_{0 \leq t \leq T}$ and $\left(p_t\right)_{0 \leq t \leq T}$ can be explicitly studied in the Gaussian case with the following proposition (proved in Appendix~\ref{appendix:proof_prop:marginal_generatives}).

 \begin{prop}[Marginals of the generative processes under Gaussian assumption]
 	\label{prop:marginal_generatives}
 	Under Gaussian assumption,
 	$\left(\ygenSDE_t\right)_{0 \leq t \leq T}$ and $\left(\ygenODE_t\right)_{0 \leq t \leq T}$
 	are Gaussian processes.
 	At each time $t$, $\pgenSDE_t$ is the Gaussian distribution $\mathcal{N}(\zero,\SigmagenSDE_t)$ with $\SigmagenSDE_t = \bSigma_{t}+e^{-2(B_{T}-B_t)}\bSigma^2_{t}\bSigma_T^{-1}(\bSigma^{-1}_T-\I)$ and $\pgenODE_t$ is the Gaussian distribution $\mathcal{N}(\zero,\SigmagenODE_t)$ with $\SigmagenODE_t = \bSigma^{-1}_T\bSigma_{t}$.
 	For all $0 \leq t \leq T$, the three covariance matrices $\bSigma_t$, $\SigmagenSDE_t$ and $\SigmagenODE_t$ share the same range.
 	Furthermore, for all $0 \leq t \leq T$,
 	\begin{equation}
 		\label{eq:W2_inequality}
 		\Wass(\pgenSDE_t,p_{t}) \leq \Wass(\pgenODE_t,p_{t})
 	\end{equation}
 	which shows that %{
  at each time $0 \leq t \leq T$ (and in particular for $t=0$ which corresponds to the desired outputs of the sampler), the SDE sampler is a better sampler than the ODE sampler when the exact score is known.
 \end{prop}
 In practice the initialization error for the SDE and ODE samplers may vary by several orders of magnitude, as shown for the CIFAR-10 example in Figure~\ref{fig:cifar_measures}.(a) (solid lines) which illustrates \Cref{eq:W2_inequality}.

 \paragraph{The discretization error.} The implementation of the SDE and the ODE implies choosing a discrete numerical scheme. We propose to study three different schemes for the SDE and the ODE, presented in Table~\ref{tab:discretization_schemes} in Appendix~\ref{appendix:table_numerical_schemes}. The classical Euler-Maruyama (EM) is used in \citep{yang_song_score_based_sde_2021_ICLR} and the exponential integrator (EI) in \citep{de_bortoli_convergence_diffusion_manifold_hypothesis_2022_TMLR} to sample from the SDE of \Cref{eq:sde_backward_generative}. The Discrete Denoising Diffusion Probabilistic Model (DDPM) \cite{ho_ddpm_2020_neurips} can be interpreted as a discretization of the backward SDE as detailed in \citep[Appendix E]{yang_song_score_based_sde_2021_ICLR}. 
Euler method, Heun's scheme and Runge-Kutta 4 (RK4) are Runge-Kutta methods with respective orders 1,2,4.  Heun's scheme is recommended in \citep{Karras_Elucidating_desing_diffusion_neurips_2022} to model the ODE of \Cref{eq:flow_ode_gen}. Under the Gaussian assumption, the discretized processes obtained using these six schemes remain Gaussian processes, and the eigenvalues of their covariance matrix can be computed numerically and recursively in order to evaluate the Wasserstein distance.
 %{
 Indeed, under the Gaussian assumption, the score is a linear operator and the discrete schemes lead to linear operations described in Table~\ref{tab:discretization_schemes} in Appendix~\ref{appendix:table_numerical_schemes}.
 Then, a Gaussian initialization for $\y_0$ provides a sequence of centered Gaussian processes $(y^{\Delta,\cdot}_k)_k$ and if $\y_0$ follows $p_T$ or $\N$, the covariance matrix %of
 $\Cov(y^{\Delta,\cdot}_k)$ 
 %commute with $\bSigma$ for $0 \leq k \leq N$. As a consequence, the covariance of $y^{\Delta,\cdot}_k$ and $\bSigma$ 
 admits the same eigenvectors as $\bSigma$ and we can use \Cref{eq:W2_Gaussian_commute_diago} to compute Wasserstein distances. %(see source code).
 Let us illustrate the computation of the eigenvalues with the EM scheme.
 %{More precisely, for each linear scheme of Table~\ref{tab:discretization_schemes} the covariance matrices of the processes over the time admit the same eigenvectors than $\bSigma$ and thus Equation~\eqref{eq:W2_Gaussian_commute_diago} is valid. }.
 Denoting $\left(\lambda_{i,t}\right)_{1 \leq i \leq d}$ the eigenvalues of $\bSigma_t$ and $\big(\lambda_{i,k}^{\Delta,\text{EM}} \big)_{1\leq i \leq d}$ the eigenvalues of the covariance matrix of the Euler-Maruyama discretization of the SDE at the $k$th step, $1 \leq k \leq N-1$, the relation verified by these eigenvalues is
 \begin{align}
 	\label{eq:EM_eigenvalues}
    \lambda_{i,k+1}^{\Delta,\text{EM}}  =  & \big(1+\Delta_t \beta_{T-t_k}  (1-\tfrac{2}{\lambda_{i,T-t_k}})\big)^2\lambda_{i,k}^{\Delta,\text{EM}} \notag \\ &+ 2\Delta_t \beta_{T-t_k},
    1 \leq i \leq d, 0 \leq k \leq N-2
 \end{align}
 with initialization $\lambda_{i,0}^{\Delta,\text{EM}} = \left\{
    \begin{array}{ll}
        1 & \mbox{if $\y_T \sim \N$ }  \\
        \lambda_{i,T} & \mbox{if $\y_T\sim p_T$ }  \\
    \end{array}
\right.$. %{
More detailed computations for EM and formulas for other schemes are presented in Appendix~\ref{appendix:computation_W2_errors}.
 For each scheme, we recursively compute the eigenvalues at each time discretization and present the observed Wasserstein distance in Figure~\ref{fig:cifar_measures}.(a). The schemes are compared at a fixed Number of score Function Evaluation (NFE). 
 We can observe that RK4 and Heun's methods provide the lower Wasserstein distance, followed by EM, EI and the Euler scheme. Note that the discrete schemes does not preserve the range of the covariance matrix, contrary to the continuous formulas. This explains the fact that the Wasserstein distance increases at the final step.

\paragraph{The truncation error.}  As discussed in \citep{yang_song_score_based_sde_2021_ICLR}, it is preferable to study the backward process on $[\varepsilon,T]$ instead of $[0,T]$ because the score is a priori not defined for $t=0$, which occurs in our case if $\bSigma$ is not invertible.
This approximation is called the truncation error. As a consequence, even without initialization error, the backward process leads to $p_{\varepsilon}$ and not $p_0$. 
Under Gaussian assumption, it is possible to explicit this error with the expression given in Proposition~\ref{prop:solution_ode_Gaussian} and~\ref{prop:solution_backward_sde_gaussian_assumption} as done in Figure~\ref{fig:cifar_measures}.(b) for both continuous and numerical solutions. 
For the standard practice truncation time $\varepsilon = 10^{-3}$ \citep{yang_song_score_based_sde_2021_ICLR,Karras_Elucidating_desing_diffusion_neurips_2022}, all numerical schemes have an error close to the corresponding continuous solution.
Using a lower $\varepsilon$ value is only relevant for the continuous SDE solution.

\setlength{\tabcolsep}{8pt} % Default value: 6pt

\begin{table}
\centering
\footnotesize
\begin{tabular}{llllllll} 
\toprule
&& \multicolumn{2}{c}{NFE = $500$}
& \multicolumn{2}{c}{NFE = $1000$}
\\
\cmidrule(lr){3-4}
\cmidrule(lr){5-6}
\cmidrule(lr){7-8}
& &
$p_T$ & $\mathcal{N}_0$ &
$p_T$ & $\mathcal{N}_0$ \\
\midrule
\multirow{2}{*}{EM}
& \multicolumn{1}{|l}{$\varepsilon = 0$}
 & 0.32
 & 0.32
 & 0.16
 & 0.16
\\
& \multicolumn{1}{|l}{$\varepsilon = 10^{-3}$}
 & 0.40
 & 0.40
 & 0.27
 & 0.27
\\
& \vspace{-.3cm} & & & & \\
\multirow{2}{*}{EI}
& \multicolumn{1}{|l}{$\varepsilon = 0$}
 & 0.30
 & 0.30
 & 0.16
 & 0.16
\\
& \multicolumn{1}{|l}{$\varepsilon = 10^{-3}$}
 & 0.41
 & 0.41
 & 0.29
 & 0.29
\\
& \vspace{-.3cm} & & & & \\
\multirow{2}{*}{DDPM}
& \multicolumn{1}{|l}{$\varepsilon = 0$}
 & 0.47
 & 0.47
 & 0.23
 & 0.23
\\
& \multicolumn{1}{|l}{$\varepsilon = 10^{-3}$}
 & 0.53
 & 0.53
 & 0.32
 & 0.32
\\
& \vspace{-.3cm} & & & & \\
\multirow{2}{*}{Euler}
& \multicolumn{1}{|l}{$\varepsilon = 0$}
 & 0.20
 & 0.26
 & 0.10
 & 0.17
\\
& \multicolumn{1}{|l}{$\varepsilon = 10^{-3}$}
 & 0.27
 & 0.32
 & 0.21
 & 0.25
\\
& \vspace{-.3cm} & & & & \\
\multirow{2}{*}{Heun}
& \multicolumn{1}{|l}{$\varepsilon = 0$}
 & -
 & -
 & -
 & -
\\
& \multicolumn{1}{|l}{$\varepsilon = 10^{-3}$}
 & 0.16
 & 0.18
 & 0.17
 & 0.19
\\
& \vspace{-.3cm} & & & & \\
\multirow{2}{*}{RK4}
& \multicolumn{1}{|l}{$\varepsilon = 0$}
 & -
 & -
 & -
 & -
\\
& \multicolumn{1}{|l}{$\varepsilon = 10^{-3}$}
 & 0.15
 & 0.17
 & 0.17
 & 0.19
\\
& \vspace{-.3cm} & & & & \\
\bottomrule
\end{tabular}
\caption{\small \label{tab:ablation_study_cifar10_extract} \textbf{Ablation study of Wasserstein errors for the CIFAR-10 Gaussian.} This table is extracted from Table~\ref{tab:ablation_study_cifar10} in Appendix~\ref{appendix:ablation_study_cifar10}. For a given discretization scheme, the table presents the Wasserstein distance associated with the truncation error for different values of $\varepsilon$. The columns $p_T$ and $\N$ show the influence of the initialization error. The columns provide the Wasserstein errors for each scheme for a given number of score function evaluation NFE. }
\end{table}

\paragraph{Ablation study.} We propose in \Cref{tab:ablation_study_cifar10_extract} an ablation study to monitor the magnitude of each error and their accumulation for various sampling schemes for the CIFAR-10 example.
In accordance with Proposition~\ref{prop:marginal_generatives}, 
the initialization error influences the ODE schemes, while the SDE schemes are not affected.
Schemes having a sufficient number of steps are not sensitive to the truncation error for $\varepsilon < 10^{-3}$, except Heun's scheme and RK4, which are unstable near to origin. 
The discretization error is the more important approximation but it becomes very low for a sufficient number of steps for stochastic schemes.  In a realistic setting, where $p_T$ is unknown, and with a truncation time $\varepsilon$, the lower Wasserstein error is provided by RK4 and Heun's method with $500$ NFE, with the classical choice $\varepsilon = 10^{-3}$. 
As in \cite{Karras_Elucidating_desing_diffusion_neurips_2022}, our conclusions lead to the choice of Heun's scheme as the go-to method.

\newlength{\wrtlambwidth}
\setlength{\wrtlambwidth}{.5\textwidth}
\setlength{\tabcolsep}{1pt} % Default value: 6pt

\setlength{\tabcolsep}{1pt} % Default value: 6pt
\begin{figure*}
          \centering
          \begin{tabular}{cc}
          \includegraphics[width = \wrtlambwidth]{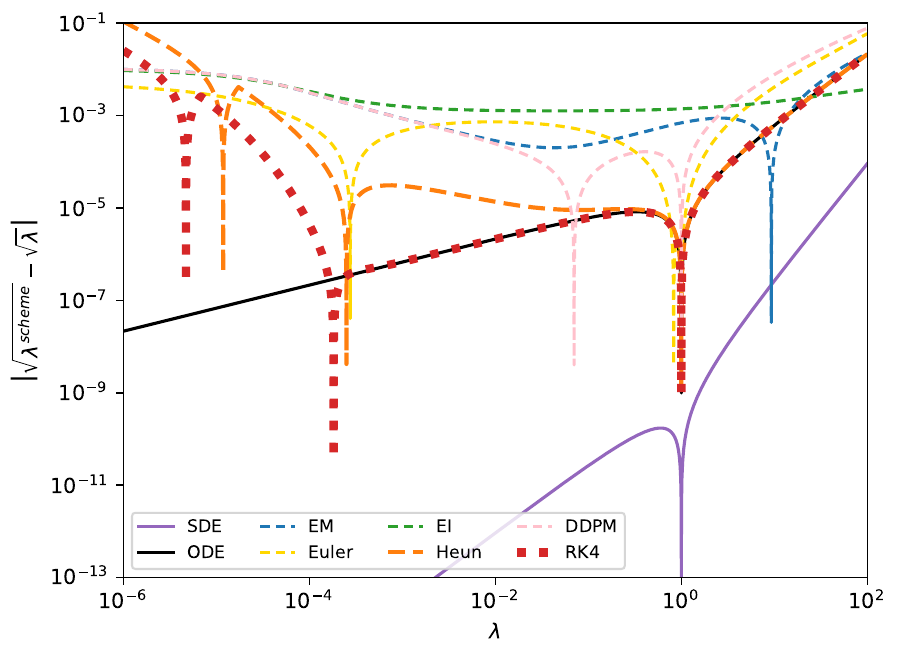}
          & \includegraphics[width = \wrtlambwidth]{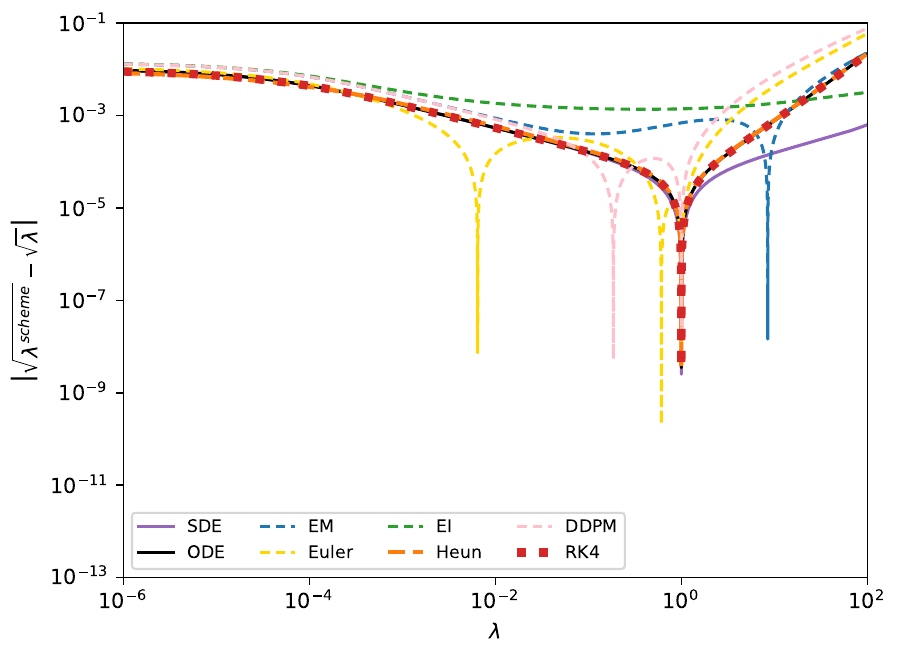} \\
          (a) Initialization error at final time
          & (b) Truncation error at final time for $\varepsilon = 10^{-3}$
          \end{tabular}
 \caption{\small \label{fig:error_wrt_lamb}
 \textbf{Eigenvalue contribution to the Wasserstein error.} The magnitude of the Wasserstein error is influenced by the eigenvalues of the covariance of the Gaussian distribution.
Left: Contribution to the Wasserstein error for the continuous equations and the discretization schemes with standard initialization $\N$.
Right: Same plot when using a truncation time $\varepsilon = 10^{-3}$.
All schemes use $1000$ NFE.
While we prove that the continuous SDE is always better than the continuous ODE (Proposition~\ref{prop:marginal_generatives}), it is not the same for the discrete schemes.
With a truncation time $\varepsilon=10^{-3}$ (see (b)), RK4 and Heun's method are nearly as good as the continuous ODE solution for high eigenvalues, which shows they are well-adapted to any Gaussian distribution.}
\end{figure*}

\paragraph{Influence of eigenvalues.}
The above observations and conclusions are observed on the CIFAR-10 Gaussian.
However, in general, they depend on the eigenvalues of the covariance matrix $\bSigma$.
Indeed, as seen in \Cref{eq:W2_Gaussian_commute_diago}, the Wasserstein distance is separable and each eigenvalue contributes to increase it.
In Figure~\ref{fig:error_wrt_lamb}, we evaluate the contribution of each eigenvalue by plotting $\lambda \mapsto |\sqrt{\lambda} - \sqrt{\lambda^\text{scheme}}|$ for each scheme.
Figure~\ref{fig:error_wrt_lamb}.(a) demonstrates that for continuous equations, the error increases with the eigenvalues, except for a strong decrease for $\lambda = 1$.
Besides, as proved in the proof of Proposition~\ref{prop:marginal_generatives} (see Appendix~\ref{appendix:proof_prop:marginal_generatives}), the error for the SDE is always lower than the error for the ODE and we can observe how tight is \Cref{eq:W2_inequality}.
Unfortunately, once discretized, the stochastic schemes are not as good as the continuous solution. We can note that all schemes exhibit a strong decrease for $\lambda=1$, except EM and EI. The other peaks depend on the choice of the parameterization of $\beta$.
The EI scheme is the most stable across the range of eigenvalues, but in the end, it is generally more costly than the others in terms of Wasserstein error. DDPM seems to be the best choice of stochastic scheme for low eigenvalues but fails for higher ones, which is largely penalized in our Gaussian CIFAR-10 example (see Figure~\ref{fig:cifar_measures}).
Without truncation time, Heun's method and RK4 fail for low eigenvalues because $\bSigma$ is not stably invertible.
However, as seen in Figure~\ref{fig:error_wrt_lamb}.(b), with a truncation time $\varepsilon=10^{-3}$, they are very close to the continuous ODE solution.
This shows that for any Gaussian distribution, these methods introduce almost no additional discretization error. Due to its simplicity, Heun's method is the preferred choice in practice.
%{
Our code allows for the evaluation of any covariance matrix and the computation of Figure~\ref{fig:cifar_measures} and Table~\ref{tab:ablation_study_cifar10_extract} (provided the eigenvalues can be computed, see the supplementary).

\section{Numerical Study of the Score Approximation}
\label{sec:score_approximation_adsn}

So far our theoretical and numerical study has been conducted under the hypothesis that the score function is known, thus discarding the evaluation of the score approximation.
In practice, for general data distribution, the score function is parameterized by a neural network trained using denoising score-matching.
This learned score function is not perfect and while theoretical studies assume the network to be close to the theoretical one (with uniform or adaptative bounds, see the discussion in \citep{de_bortoli_convergence_diffusion_manifold_hypothesis_2022_TMLR}),
such an hypothesis is hard to check in practice, especially in our non compact setting.
Thus, we propose in this section to train a diffusion models on a Gaussian distribution and evaluate numerically the impact of the score approximation.

\paragraph{The Gaussian ADSN distribution for microtextures.} So far our running example was the CIFAR-10 Gaussian but we will now turn to another example that produces visually interesting images, namely Gaussian microtextures.
We consider the asymptotic discrete spot noise (ADSN) distribution~\citep{Galerne_Gousseau_Morel_random_phase_textures_2011} associated with an RGB texture $\U\in\R^{3\times \OMN}$ which is defined as the stationary Gaussian distribution whose covariance equals the autocorrelation of $\U$.
More precisely, this distribution is sampled using convolution with a white Gaussian noise~\citep{Galerne_Gousseau_Morel_random_phase_textures_2011}:
Denoting $m\in\R^3$ the channelwise mean of $\U$ and $\bt_c = \frac{1}{\sqrt{MN}}(\U_c-m_c)$, $1 \leq c \leq 3$, its associated \emph{texton}, for $\w \sim \N$ of size $M \times N$
the channelwise convolution $\x = m + \bt \star \w \in \R^{3 \times \OMN}$ follows $\ADSN(\U)$. This distribution is the Gaussian $\mathcal{N}(m,\bSigma)$. 
To deal with zero mean Gaussian, adding the mean $m$ is considered as a post-processing to visualize samples and we study $\mathcal{N}(\zero,\bSigma)$. 
The matrix $\bSigma$ is a well-known convolution matrix \citep{Ferradans_Static_dynamic_texture_0T_2013}, its eigenvectors and associated eigenvalues can be computed in the Fourier domain, as done in Appendix~\ref{appendix:covariance_ADSN_eigen}. 
$\bSigma$ admits the eigenvalues $\lambda^{\xi,\ADSN}_1 = |\widehat{\bt}_1|^2(\xi)+|\widehat{\bt}_2|^2(\xi)+|\widehat{\bt}_3|^2(\xi), \xi \in \OMN$ and $0$ with multiplicity $2MN$ and we can conduct the same analysis as before (see Appendix~\ref{appendix:ADSN_error}). 
To evaluate if a set of $N_\text{samples}$ sampled images is close to the ADSN distribution $\pdata$, we evaluate a problem-specific empirical Wasserstein distance: Supposing that the $N_\text{samples}$ are drawn from a Gaussian distribution $\pemp = \mathcal{N}(\zero,\G)$ such that $\G$ admits the same eigenvectors as $\bSigma$, we compute
\begin{align}
	\label{eq:W2emp}
	\Wemp(\pemp,\pdata)^2 = & \sum_{\xi \in \OMN}\left(\sqrt{\lambda^{\xi,\text{emp.}}_1}-\sqrt{\lambda^{\xi,\ADSN}_1}\right)^2 \notag \\
    & + \lambda^{\xi,\text{emp.}}_2 + \lambda^{\xi,\text{emp.}}_3
\end{align}
where $(\lambda^{\xi,\text{emp.}}_i)_{\xi \in \OMN,1 \leq i \leq 3}$ are estimators of the eigenvalues of $\G$ given in Appendix~\ref{appendix:covariance_ADSN_W2}.
\newlength{\samplewidth}
\setlength{\samplewidth}{.14\textwidth}
\setlength{\tabcolsep}{1pt} % Default value: 6pt
\begin{figure}
	\centering
	\small
	\begin{tabular}{cccc}
	\vspace{0.003\textwidth}
	\parbox[t]{3mm}{\rotatebox[origin=c]{90}{Original $\U$}}   & \raisebox{-.5\height}{\includegraphics[width=\samplewidth]{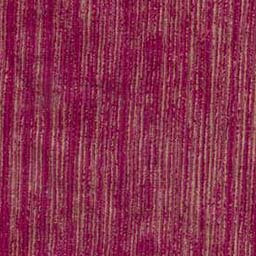}} & &     \\
	\vspace{0.003\textwidth}
	\parbox[t]{3mm}{\rotatebox[origin=c]{90}{ADSN $\quad$}} &
	\raisebox{-.5\height}{\includegraphics[width=\samplewidth]{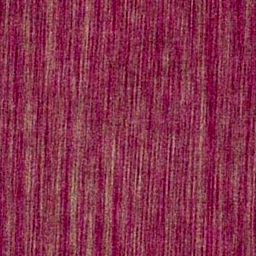}} &
	\raisebox{-.5\height}{\includegraphics[width=\samplewidth]{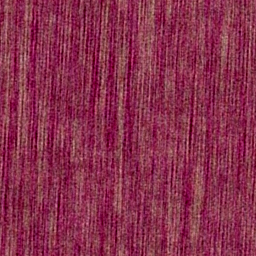}} &
	\raisebox{-.5\height}{\includegraphics[width=\samplewidth]{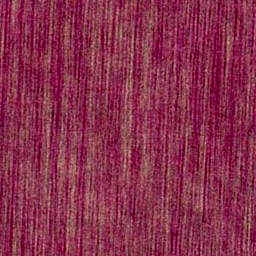}}  \\
	\vspace{0.003\textwidth}
	\parbox[t]{4mm}{\rotatebox[origin=c]{90}{$p_\theta^\text{EM}$ $\quad$}} &
	\raisebox{-.5\height}{\includegraphics[width=\samplewidth]{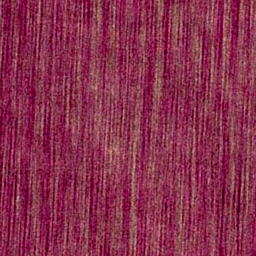}} &
	\raisebox{-.5\height}{\includegraphics[width=\samplewidth]{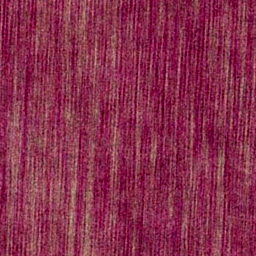}} &
	\raisebox{-.5\height}{\includegraphics[width=\samplewidth]{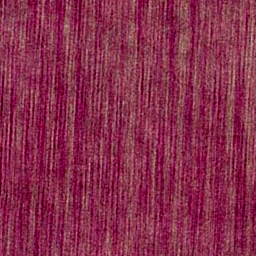}}  \\
	\vspace{0.003\textwidth}
	\parbox[t]{4mm}{\rotatebox[origin=c]{90}{$p_\theta^\text{Heun}$ $\quad$}} &
	\raisebox{-.5\height}{\includegraphics[width=\samplewidth]{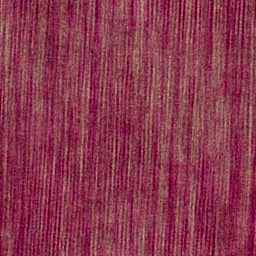}} &
	\raisebox{-.5\height}{\includegraphics[width=\samplewidth]{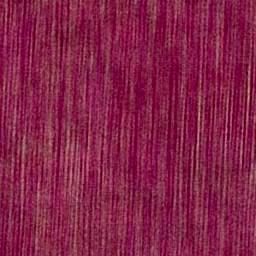}} &
	\raisebox{-.5\height}{\includegraphics[width=\samplewidth]{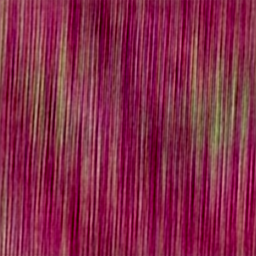}}   \\
	\end{tabular}
\caption{\small \label{fig:samples} \textbf{Texture samples generated with the learned score.}
First row: original image $\U$.
%and its DFT modulus (for all DFT modulus we display the sum of the DFT modulus of the three color channels and apply a logarithmic contrast change).
Second row: three samples of $\ADSN(\U)$ 
%with their associated DFT moduli.
Third and fourth row: Samples generated with the learned score with EM and Heun's discretization schemes.% and their associated DFT moduli.
While both schemes use the same learned score function, the generation with Heun's scheme can produce out-of-distribution samples, as seen with the third sample.}
\end{figure}

\setlength{\tabcolsep}{6pt} % Default value: 6pt
\begin{table*}[h]
%\small
\centering
{%\small
\begin{tabular}{@{}lccccc@{}}
\toprule
 & \multicolumn{3}{c}{Exact score distribution} & \multicolumn{2}{c}{Learned score distribution} \\
\cmidrule(lr){2-4}
\cmidrule(lr){5-6}
 $p$ &  $\Wass$($p$,$\pdata $) $\downarrow$ &  $\Wemp$($\pemp$,$\pdata $) $\downarrow$ & FID($\pemp$,$\pdataemp$) $\downarrow$ & $\Wemp$($\pemp_\theta$,$\pdataemp $) $\downarrow$ &FID($\pemp_\theta$,$\pdataemp$) $\downarrow$  \\
\midrule
EM & 5.16 &5.1630$\pm$7E-5  & 0.0891$\pm$8E-4   & 15.6 & \textcolor{white}{0}1.02 \\
Heun & 3.73 &3.7323$\pm$2E-4 & 0.0447$\pm$6E-4  & 56.7 & 19.4\textcolor{white}{8}\\
\bottomrule
\end{tabular}}
\caption{\footnotesize \label{tab:W2_FID} \textbf{Numerical evaluation of the score approximation for a Gaussian microtexture model.} For two schemes, the EM discretization of the backward SDE and Heun's method associated with the flow ODE, the table shows the Wasserstein distance and FID for theoretical and learned distributions. The theoretical $\Wass$ value is computed with explicit formulas, as done in Table~\ref{tab:ablation_study_ADSN}. The FID and empirical $\Wass$ w.r.t the theoretical distribution are computed on $25$ samplings of 50K images while only one sampling of 50K images is drawn for the parametric distributions (to limit computation time).}
\end{table*}
\paragraph{Learning the score function.} We train the network using the code\footnote{Code available at \url{https://github.com/yang-song/score_sde_pytorch}} associated with the paper \cite{yang_song_score_based_sde_2021_ICLR}. We choose the architecture of DDPM, which is a U-Net described in \cite{ho_ddpm_2020_neurips}, with the parameters proposed for the dataset CelebaHQ256 to deal with the 256$\times$256 ADSN model associated with the top-left image of Figure~\ref{fig:samples}. We use the training procedure corresponding to \emph{DDPM cont.} in \cite{yang_song_score_based_sde_2021_ICLR}. 
$\beta$ is linear from 0.05 to 10 with $T=1$. 
We train over 1.3M iterations, and we generate at each iteration a new batch of ADSN samples.
We implement the stochastic EM and deterministic Heun schemes replacing the exact score by its learned version with $N = 1000$ steps and a truncation time $\varepsilon = 10^{-3}$. We name $p_\theta^{\text{EM}}$ and $p_\theta^{\text{Heun}}$, the corresponding distributions and present samples in Figure~\ref{fig:samples}. 
Both distributions accumulate the four error types.

\paragraph{Evaluation of the score approximation.} It is not possible to compute theoretically the Wasserstein distance between $\pdata = \ADSN(\U)$ and $p_\theta^{\text{EM}},p_\theta^{\text{Heun}}$ due to the non-linearity of the learned score. To compute an empirical Wasserstein error between it, we use \Cref{eq:W2emp}. 
Let us clarify that this approximation underestimates the real Wasserstein distance since it wrongly assumes that the distributions $p_\theta^{\text{EM}}$, $p_\theta^{\text{Heun}}$ are Gaussian with a covariance matrix diagonalizable in the same basis than the covariance matrix $\bSigma$ of $\ADSN(\U)$.
We complete this dedicated empirical measure with the standard FID.
These metrics are reported in Table~\ref{tab:W2_FID} where for theoretical distributions that are fast to sample we add the standard deviations computed on $25$ different $50k$-samplings. 
For this Gaussian distribution, the score approximation is by far the most impactful source of error, which is in accordance with previous works~\cite{chen_sampling_as_easy_as_learning_Score_ICLR_2023,de_bortoli_et_al_diffusion_schrodinger_bridge_2021_neurips}.
We observe that the stochastic EM sampling is more resilient to score approximation than the deterministic Heun's scheme, resulting in out-of-distribution samples (Figure~\ref{fig:samples}).
We may explain this behavior by recalling the results of Proposition~\ref{prop:marginal_generatives} that shows that SDE solutions are less sensitive to initialization errors than ODE.
Indeed, adding noise at each iteration tends to mitigate the accumulated errors, and score approximation may be considered as some initialization error occurring at each step.

\section{Discussion and Limitations}
\label{sec:discussion_limitations}
%{
The main limitation of our work is that our theoretical and numerical results are limited to Gaussian distributions.
Resorting to diffusion models for sampling Gaussian distributions is not necessary in practice, rather we use Gaussian distributions as a test case family to provide insight on diffusion models.

A natural extension of this work is to compute error types for more complex distributions (e.g multimodal) such as Gaussian mixtures models (GMM). 
However, generalizing our results for these more complex distributions one faces three main difficulties. 
First, to the best of our knowledge, we are unable to derive exact solutions to the backward SDE or the flow ODE under GMM assumption, even though the score has a known analytical expression \citep{Zach_Gaussian_miwture_diffusion_models_2024, Zach_explicit_diffusion_Gaussian_mixture, shah_learning_mixture_Gaussians_Neurips_2023}.
Another key feature of this study is to evaluate exactly the Wasserstein error by using \Cref{eq:W2_Gaussian_commute_diago}, strongly relying on the Gaussian assumption. 
A closed-form of the Wasserstein distance between two GMMs is not known, 
leading to alternative distance definitions for such models~\citep{delon_desolneux_wasserstein_type_distance_GMM}. 
Furthermore, there is no guarantee that the processes obtained by discretizing the backward SDE follow a known distribution such as a GMM.
% Even if a study of the score approximation by a network could be led, it implies then to apply approximate Wasserstein distance. 
Hence, to compare the distributions generated in practice with exact solutions of time continuous equations under GMM assumption, as we do for the Gaussian case, one should solve open theoretical problems.
% The objective of our article is to adopt another point of view than the state of the art work by proposing an exact evaluation of the different error types and not only an upper estimation%in the convergence of diffusion models
% and therefore its extension to more general distributions represents a complete work to carry out in the future.

\section{Conclusion}

By restricting the analysis of diffusion models to the specific case of Gaussian distributions, we were able to derive exact solutions for both the backward SDE and its associated probability flow ODE.
%{
We demonstrated that regarding the initialization error, the SDE sampler is more resilient than the ODE sampler for Gaussian distributions.
Additionally, we characterized the discrete Gaussian processes arising when discretizing these equations.
This allowed us to provide exact Wasserstein errors for the initialization error, the discretization error, and the truncation error as well as any of their combinations.
This theoretical analysis led to conclude that Heun's scheme is the best numerical solution, in accordance with empirical previous work~\citep{Karras_Elucidating_desing_diffusion_neurips_2022}.

To conclude our work we conducted an empirical analysis with a learned score function using standard architecture which showed that the score approximation error may be the most important one in practice and that this error has more impact on the deterministic Heun's scheme than the stochastic Euler-Maruyama scheme. 
This suggests that assessing the quality of learned score functions is an important research direction for future work.

\section*{Impact Statement}

This paper presents work whose goal is to advance the field of Machine Learning. There are many potential societal consequences of our work, none which we feel must be specifically highlighted here.

 \bigskip
 {\footnotesize \noindent\textbf{Acknowledgements:} The authors acknowledge the support of the project MISTIC (ANR-19-CE40-005) and the CaSciModOT Centre de Calcul Scientifique de la Région Centre Val de Loire for providing computer facilities.}

\bibliography{Diffusion}
\bibliographystyle{icml2025}

\appendix

\onecolumn
\section{Characterization of Gaussian distributions through diffusion models}
\label{appendix:prop:Gaussian_linear_score}
\label{appendix:Gaussian_linear_score}

The following proposition shows that our Gaussian assumption occurs if and only if the score function is linear.

\begin{prop}
\label{prop:Gaussian_linear_score}
	The three following propositions are equivalent:
	\begin{enumerate}[(i)]
		\item $\x_0 \sim \mathcal{N}(\zero,\bSigma)$ for some covariance $\bSigma$.
		\item $\forall t > 0, \nabla_x \log p_t(\x)$ is linear w.r.t $\x$.
		\item $\exists t> 0, \nabla_x \log p_t(\x)$ is linear w.r.t $\x$.
	\end{enumerate}
	In this case, for $t > 0$, $\nabla_{\x} \log p_t(\x) = -\bSigma^{-1}_{t} \x$, with $\bSigma_{t}$ defined in Proposition~\ref{prop:solution_forward_sde}.
\end{prop}

\begin{proof}
$(ii) \Rightarrow (iii)$ is clear.

If $(i)$, for $t > 0$, $p_t(\x) = C_t \exp\left(-\frac{1}{2}\x^T \bSigma^{-1}_{t}\x\right)$. Consequently, $\nabla_{\x}\log p_t(\x) = -\bSigma^{-1}_{t}\x$ and $(i) \Rightarrow (ii)$

If $(iii)$, there exists $A$ such that $\nabla_{\x}\log p_t(\x) = A\x$. Consequently, $p_t(\x) = C_t\exp(-\frac{1}{2}\x^TA\x)$ and $\x_t$ is Gaussian. This provides that $\x_0 = e^{B_{t}}\x_t -\etab_t$ is Gaussian and $(iii) \Rightarrow (i)$.
\end{proof}

\section{Proofs of Section~\ref{sec:exact_sde_ode_solutions}}
 \subsection{Proposition~\ref{prop:solution_forward_sde}: Solution of the forward SDE}

 \label{appendix:proof_prop:solution_forward_sde}

 We aim at solving:

\begin{equation}
 	d\x_t  =-\beta_t \x_t dt + \sqrt{2\beta_t}d\w_t, \quad  \x_0 \sim \pdata.
\end{equation}

 By considering $\z_t = e^{B_{t}} \x_t $ where $B_{t} = \int_0^t \beta_sds$,

\begin{equation}
	d\z_t
	= \beta_{t} e^{B_{t}}\x_t dt  +  e^{B_{t}}d\x_t
	=  \beta_{t} e^{B_{t}}\x_t dt +  e^{B_{t}}(-\beta_t \x_t dt + \sqrt{2\beta_t}d\w_t)
	= \sqrt{2\beta_{t}}e^{B_{t}}d\w_t.
\end{equation}

Consequently, for $0\leq t \leq T$,

\begin{equation}
\z_t = \z_0 + \int_0^t\sqrt{2\beta_{s}}e^{B_{s}}d\w_s, \z_0 = e^{B_{0}}\x_0 = \x_0
\end{equation}

and for $0\leq t \leq T$,

\begin{equation}
%\begin{aligned}
	\x_t
	 = e^{-B_{t}}\z_t %\\
     = e^{-B_{t}}\x_0 + e^{-B_{t}}\int_0^te^{B_{s}} \sqrt{2\beta_{s}}d\w_s  %\\
     = e^{-B_{t}}\x_0 + \etab_t. %\\
%\end{aligned}
\end{equation}

By Itô's isometry (see e.g \cite{oksendal_SDE_with_applications_springer_2010}),

\begin{equation}
%\begin{aligned}
	\text{Var}\left( \int_0^te^{B_{s}} \sqrt{2\beta_{s}}d\w_s\right)
	 = \int_{0}^t  2\beta_{s}e^{2B_{s}}ds %\\
	 =[e^{2B_{s}}]_{0}^t % \\
	 = e^{2B_{t}}-e^{2B_{0}} %\\
	 = e^{2B_{t}}-1 %\\
%\end{aligned}
\end{equation}

which provides the covariance matrix of $\etab_t$:

\begin{equation}
%\begin{aligned}
	\Cov\left(\etab_t \right)
	 = e^{-2B_{t}}(e^{2B_{t}}-1)\I %\\
	 = \left(1-e^{-2B_{t}}\right)\I.
%	\end{aligned}
\end{equation}

Because $\x_0$ and $\etab_t$ are independent, $\bSigma_{t} =  e^{-2B_{t}}\bSigma+\left(1-e^{-2B_{t}}\right)\I$.

And,
\begin{equation}
%\begin{aligned}
	d\bSigma_t
	 = -2\beta_{t} e^{-2B_{t}}(\bSigma - \I)dt
	 =  -2\beta_{t}\left[\bSigma_{t} -\I\right]dt
	 %= -(2\beta_{t} \I -2\beta_{t}\bSigma^{-1}_{t})\bSigma_{t}dt.
%\end{aligned}
\end{equation}

  \subsection{Proposition~\ref{prop:solution_backward_sde_gaussian_assumption}: Solution of the backward SDE under Gaussian assumption}

 \label{appendix:proof_prop:solution_backward_sde}

 We aim at solving

\begin{equation}
 d\y_t = \beta_{T-t}(\y_t - 2 \bSigma^{-1}_{T-t}\y_t)dt + \sqrt{2\beta_{T-t}}d\w_t, \quad  0 \leq t \leq T
\end{equation}

Denoting $C_{t} = \int_{0}^t \beta_{T-s}ds$, by considering $\z_t = \bSigma^{-1}_{T-t} e^{C_{t}}\y_t$,

{\small
\begin{align}
	d\z_t
	& =  e^{C_{t}}\bSigma^{-1}_{T-t}d\y_t +  e^{C_{t}}d[\bSigma^{-1}_{T-t}]\y_t + \beta_{T-t} \z_t dt \\
	%& =  e^{C_{t}}\bSigma^{-1}_{T-t}\left[(\beta_{T-t}(\y_t   - 2\bSigma^{-1}_{T-t}\y_t )dt)\right] - (2\beta_{T-t} \I - 2\beta_{T-t}\bSigma^{-1}_{T-t}
%)\z_tdt + \beta_{T-t} \z_t dt, \text{(Equation~\eqref{eq:ode_covariance})} \\
    %& = \bSigma^{-1}_{T-t} e^{C_{t}}d\y_t - 2\beta_{T-t}\z_tdt + 2\beta_{T-t}\bSigma^{-1}_{T-t}
    %        \z_tdt + \beta_{T-t} \z_t dt \\
    & =  e^{C_{t}}\bSigma^{-1}_{T-t}d\y_t -  e^{C_{t}}\bSigma^{-1}_{T-t}d[\bSigma_{T-t}]\bSigma^{-1}_{T-t}\y_t + \beta_{T-t} \z_t dt \text{ by derivative of the inverse matrix}\\
    & =  e^{C_{t}}\bSigma^{-1}_{T-t}\left[\beta_{T-t}(\y_t - 2 \bSigma^{-1}_{T-t}\y_t)dt + \sqrt{2\beta_{T-t}}d\w_t\right] -2\beta_{T-t}  e^{C_{t}}\bSigma^{-1}_{T-t}\left[\bSigma_{T-t} -\I\right]\bSigma^{-1}_{T-t}\y_tdt + \beta_{T-t} \z_t dt \\
    & \text{(using \Cref{eq:ode_covariance})}\\
    & = \left[\bSigma^{-1}_{T-t} e^{C_{t}}\beta_{T-t}(\y_t   - 2\bSigma^{-1}_{T-t}\y_t ) - \beta_{T-t}\z_t + 2\beta_{T-t}\bSigma^{-1}_{T-t}
            \z_t \right]dt  +   \sqrt{2\beta_{T-t}}e^{C_{t}}\bSigma^{-1}_{T-t}d\w_t \\
    & =\beta_{T-t}(\I   - 2\bSigma^{-1}_{T-t})\z_t dt  - \beta_{T-t}\z_tdt + 2\beta_{T-t}\bSigma^{-1}_{T-t}
            \z_tdt +  e^{C_{t}}\sqrt{2\beta_{T-t}}\bSigma^{-1}_{T-t}d\w_t \\
    & =  \sqrt{2\beta_{T-t}}e^{C_{t}}\bSigma^{-1}_{T-t}d\w_t. \\
\end{align}
}

Consequently,

\begin{equation}
	\z_t
	 = \z_0 + \int_0^t\sqrt{2\beta_{T-s}}e^{C_{s}}\bSigma^{-1}_{T-s} d\w_s
	= \bSigma^{-1}_{T}\y_0 + \int_0^t\sqrt{2\beta_{T-s}}e^{C_{s}}\bSigma^{-1}_{T-s}d\w_s.
\end{equation}

 And,
\begin{equation}
 %\begin{aligned}
 	\y_t
 	 = e^{-C_{t}}\bSigma_{T-t}\z_t %\\
 	= e^{-C_{t}}\bSigma_{T-t}\bSigma^{-1}_{T}\y_0 + e^{-C_{t}}\bSigma_{T-t}\int_0^t\bSigma^{-1}_{T-s} e^{C_{s}}\sqrt{2\beta_{T-s}}d\w_s.
% \end{aligned}
\end{equation}

 Finally,

\begin{equation}
 	\y_t =  e^{-C_{t}}\bSigma_{T-t}\bSigma^{-1}_{T}\y_0 + \bxi_t
 	\quad \text{with}\quad
 	 \bxi_t = e^{-C_{t}}\bSigma_{T-t}\int_0^t\bSigma^{-1}_{T-s} e^{C_{s}}\sqrt{2\beta_{T-s}}d\w_s.
\end{equation}

 By the multidimensional Itô's isometry,

\begin{equation}
	\Cov(\int_0^t\bSigma^{-1}_{T-s} e^{C_{s}}\sqrt{2\beta_{T-s}}d\w_s)
	 = 2\int_0^t e^{2C_{s}}\beta_{T-s}\bSigma^{-2}_{T-s}ds.
\end{equation}

Now, remark that for $A_s = e^{2C_{s}}\bSigma^{-1}_{T-s}$,

\begin{align}
	dA_s
	& = 2\beta_{T-s} A_s ds + e^{2C_{s}}d\left[\bSigma_{T-s}^{-1}\right] \\
	& = 2\beta_{T-s} A_s  ds - 2\beta_{T-s}e^{2C_{s}}\left[\I-\bSigma^{-1}_{T-s}\right]\bSigma^{-1}_{T-s}ds \text{ (using \Cref{eq:ode_covariance})}\\
	& = 2e^{2C_{s}}\beta_{T-s}\bSigma^{-2}_{T-s}ds.
\end{align}

\begin{equation}
%\begin{aligned}
 \Cov\left(\int_0^t\bSigma^{-1}_{T-s} e^{C_{s}}\sqrt{\beta_{T-s}}d\w_s\right)
	 = \int_0^t d A_s  %\\
	 = \left[A_s\right]_0^t% \\
	 = e^{2C_{t}}\bSigma^{-1}_{T-t} - \bSigma^{-1}_{T}.
%\end{aligned}
\end{equation}

 Finally, $\Cov(\bxi_t) = \bSigma_{T-t}^2\left( \bSigma^{-1}_{T-t} - e^{-2C_{t}}\bSigma^{-1}_{T}\right) = \bSigma_{T-t} - e^{-2C_{t}}\bSigma_{T-t}^2\bSigma^{-1}_{T}$

 We have the final formula considering:
\begin{equation}
 %\begin{aligned}
 	C_{t}
 	 = \int_0^t \beta_{T-s}ds %\\
 	 = \int_{T-t}^T \beta_{x}dx %\text{ with $x = T-s$} %\\
 	 = \int_{0}^T \beta_{x}dx -\int_{0}^{T-t} \beta_{x}dx% \\
 	 = B_T - B_{T-t}
 %\end{aligned}
\end{equation}

 that provides

 \begin{equation}
 	\Cov(\y_t) = \bSigma_{T-t} + e^{-2(B_{T} - B_{T-t})}\bSigma_{T-t}^2\bSigma^{-1}_{T}\left(\bSigma^{-1}_{T-t}\Cov(\y_0)\bSigma^{-1}_{T}\bSigma_{T-t}-\I \right).
  \end{equation}

In particular, if $\Cov(\y_0)$ and $\bSigma$ commute,

   \begin{equation}
 	\Cov(\y_t) = \bSigma_{T-t} + e^{-2(B_{T} - B_{T-t})}\bSigma_{T-t}^2\bSigma^{-1}_{T}\left(\bSigma^{-1}_{T}\Cov(\y_0)-\I \right).
  \end{equation}

 \subsection{Proposition~\ref{prop:solution_ode_Gaussian}: Solution of the ODE probability flow under Gaussian assumption}
 \label{appendix:proof_prop:solution_ode_Gaussian}

 As done in \cite{understanding_DDPM_through_OT_Khrulkov_ICLR_2023}, the matrix $\bSigma^{1/2}_{t}$ admits a derivative which is $d\left[\bSigma_{t}^{1/2}\right] = \frac{1}{2}d\bSigma_t\bSigma^{-1/2}_t$ because it is diagonalisable. Let us check that 
 \begin{equation}
     \y_t = \bSigma^{-1/2}_T\bSigma_{T-t}^{1/2}\y_0
     \label{eq:probability_flow_solution_gaussian_case_appendix}
 \end{equation}
 is solution of the ODE of \Cref{eq:flow_reverse_ode}:

 \begin{align}
 	d\y_t
 	& = -\bSigma^{-1/2}_T\frac{1}{2}d\bSigma_{T-t}\bSigma^{-1/2}_{T-t}\y_0 \\ %dt \\
 	& = \bSigma^{-1/2}_T\left[\beta_{T-t} \bSigma_{T-t} - \beta_{T-t}\I\right]\bSigma^{-1/2}_{T-t}\y_0 dt \quad (\text{using \Cref{eq:ode_covariance}})\\
  & = \left[\beta_{T-t} \bSigma_{T-t} - \beta_{T-t}\I\right]\bSigma^{-1}_{T-t}\bSigma^{-1/2}_T\bSigma^{1/2}_{T-t}\y_0 dt \quad (\text{by commutativity})\\
 	& = \left[\beta_{T-t}  - \beta_{T-t}\bSigma^{-1}_{T-t}\right]\y_t dt \\
 	& = \left[\beta_{T-t}  + \beta_{T-t}\nabla_{\y}\log p_{T-t}(\y_t)\right]\y_t dt.
 \end{align}

%{
Let us discuss the link between this solution and OT. 
The formula of OT map between two centered Gaussian distributions $\mathcal{N}(\zero,\bSigma_1)$ and $\mathcal{N}(\zero,\bSigma_2)$ is well known. In \cite{Peyre_book_OT}, the authors give the linear map (affine when the distributions are not centered) $\boldsymbol{T} : \X \mapsto \A\X$ with
\begin{equation}
    \A =  \bSigma_1^{-\frac{1}{2}} \left(\bSigma_1^{\frac{1}{2}} \bSigma_2 \bSigma_1^{\frac{1}{2}} \right)^{\frac{1}{2}} \bSigma_1^{-\frac{1}{2}}. 
\end{equation}
When $\bSigma_1$ and $\bSigma_2$ commute, this expression simplifies to:
\begin{equation}
    \A = \bSigma_1^{-1/2}\bSigma_2^{1/2}.
\end{equation}
We showed that the solution (\Cref{eq:probability_flow_solution_gaussian_case_appendix}) of the backward probability flow in the finite interval $[0,t]$, with $0\leq t \leq T$, corresponds to applying to the initial point $\y_0$ the linear map
\begin{equation}
    \A = \bSigma_T^{-\frac{1}{2}} \bSigma_{T-t}^{\frac{1}{2}},
\end{equation}
that is, the OT map between $p_T = \mathcal{N}(\zero, \bSigma_T)$ and
$p _{T-t} = \mathcal{N}( \zero, \bSigma_{T-t} )$.

%{
Let us now derive the covariance matrix of the solution, which characterises a Gaussian distribution.

\begin{equation}
	\Cov(\y_t) = \bSigma^{-1/2}_T\bSigma_{T-t}^{1/2}\Cov(\y_0)\bSigma^{-1/2}_{T-t}\bSigma_{T}^{1/2}.
\end{equation}

In particular, if $\Cov(\y_0)$ and $\bSigma$ commute,

\begin{equation}
	\Cov(\y_t) = \bSigma^{-1}_T\bSigma_{T-t}\Cov(\y_0).
\end{equation}

 \subsection{Proof of Proposition~\ref{prop:marginal_generatives}}
 \label{appendix:proof_prop:marginal_generatives}

    For $0 \leq t \leq T$, denoting $\left(\lambda_{i,t}\right)_{1\leq i \leq d}$ the eigenvalues of $\bSigma_{t}$, the eigenvalues of $\SigmagenSDE_t = \Cov(\ygenSDE_{T-t})$ are

\begin{equation}
 	\tilde{\lambda}_{i,t}
 	 = \lambda_{i,t} + e^{-2(B_T - B_{t})}\lambda_{i,t}^2\frac{1}{\lambda_{i,T}}\left(\frac{1}{\lambda_{i,T}}-1\right),\quad  i=1, \ldots, d.
\end{equation}

 and the eigenvalues of $\SigmagenODE_t = \Cov(\ygenODE_{T-t})$ are

\begin{equation}
 	\widehat{\lambda}_{i,t}
 	= \frac{\lambda_{i,t}}{\lambda_{i,T}},\quad  1 \leq i \leq d.
\end{equation}

  Consequently, $\Wass(p_{t},\pgenSDE_t)$ is the sum of the squares of all:

\begin{equation}
 	\sqrt{\lambda_{i,t}} - \sqrt{\tilde{\lambda}_{i,t}}
 	= \sqrt{\lambda_{i,t}}\left(1-\sqrt{1+e^{-2(B_T - B_{t})}\lambda_{i,t}\frac{1}{\lambda_{i,T}}\left(\frac{1}{\lambda_{i,T}}-1\right)}\right).
\end{equation}

  Similarly, $\Wass(p_{t},\pgenODE_t)$ is the sum of the squares of all:

 \begin{align}
 	\sqrt{\lambda_{i,t}} - \sqrt{\widehat{\lambda}_{i,t}}
 	 = \sqrt{\lambda_{i,t}}\left(1-\sqrt{\frac{1}{\lambda_{i,T}}}\right) \\
 	 = \sqrt{\lambda_{i,t}}\left(1-\sqrt{1+\left(\frac{1}{\lambda_{i,T}}-1\right)}\right).
 \end{align}

Let us now compare individually these differences.

 \begin{align}
 	\frac{e^{-2(B_T - B_{t})}\lambda_{i,t}\frac{1}{\lambda_{i,T}}\left(\frac{1}{\lambda_{i,T}}-1\right)}{\frac{1}{\lambda_{i,T}}-1}
 	& = e^{-2(B_T - B_{t})}\frac{\lambda_{i,t}}{\lambda_{i,T}} \\
 	& = e^{-2(B_T - B_{t})}\frac{e^{-2B_{t}}(\lambda_i-1)+1}{e^{-2B_{T}}(\lambda_i-1)+1} \\
 	& = \frac{(\lambda_i-1)+e^{2B_{t}}}{(\lambda_i-1)+e^{2B_{T}}} \\
 	& < 1.
 \end{align}

 \textbf{Case 1: $0 < \lambda_i < 1$ and $t > 0$}

 In this case, $\lambda_{i,T} < 1$ and:

\begin{equation}
0 < e^{-2(B_T - B_{t})}\lambda_{i,t}\frac{1}{\lambda_{i,T}}\left(\frac{1}{\lambda_{i,T}}-1\right) < \frac{1}{\lambda_{i,T}}-1.
\end{equation}

Thus,

 \begin{align}
 	\left|\sqrt{\lambda_{i,t}} - \sqrt{\tilde{\lambda}_{i,t}}\right|
 	& = \sqrt{\tilde{\lambda}_{i,t}}-\sqrt{\lambda_{i,t}} \\
 	& = \sqrt{\lambda_{i,t}}\left(\sqrt{1+e^{-2(B_T - B_{t})}\lambda_i^{t}\frac{1}{\lambda_{i,T}}\left(\frac{1}{\lambda_{i,T}}-1\right)}-1\right) \\
 	& < \sqrt{\lambda_{i,t}}\left(\sqrt{1+\left(\frac{1}{\lambda_{i,T}}-1\right)}-1\right) \\
 	& = \sqrt{\widehat{\lambda}_{i,t}}-\sqrt{\lambda_{i,t}} \\
 	& = \left|\sqrt{\lambda_{i,t}} - \sqrt{\widehat{\lambda}_{i,t}}\right|.
 \end{align}

\textbf{Case 2: $\lambda_i = 0$ and $t = 0$.}

In this case, for  $1 \leq i \leq d$, $\widehat{\lambda}_{i,T} = \tilde{\lambda}_{i,T} = 0$.

\textbf{Case 3: $\lambda_i = 1$.}

In this case, for  $1 \leq i \leq d$, $\widehat{\lambda}_{i,t} = \tilde{\lambda}_{i,t} = 1$.

 \textbf{Case 4: $1 < \lambda_i$.}

  In this case, $\lambda_{i,T} \geq 1$, and
  $
 	\frac{e^{-2(B_T - B_{t})}\lambda_{i,t}\frac{1}{\lambda_{i,T}}\left(\frac{1}{\lambda_{i,T}}-1\right)}{\frac{1}{\lambda_{i,T}}-1} = e^{-2(B_T - B_{t})}\frac{\lambda_{i,t}}{\lambda_{i,T}} < 1
$ provides

 \begin{equation}
e^{-2(B_T - B_{t})}\lambda_{i,t}\frac{1}{\lambda_{i,T}}\left(\frac{1}{\lambda_{i,T}}-1\right) >  \frac{1}{\lambda_{i,T}}-1.
 \end{equation}
 Finally,

 \begin{align}
 	\left|\sqrt{\lambda_{i,t}} - \sqrt{\tilde{\lambda}_{i,t}}\right|
 	& = \sqrt{\lambda_{i,t}}-\sqrt{\tilde{\lambda}_{i,t}} \\
 	& = \sqrt{\lambda_{i,t}}\left(1-\sqrt{1+e^{-2(B_T - B_{t})}\lambda_{i,T}\frac{1}{\lambda_{i,T}}\left(\frac{1}{\lambda_{i,T}}-1\right)}\right) \\
 	& < \sqrt{\lambda_{i,t}}\left(1-\sqrt{1+\left(\frac{1}{\lambda_{i,T}}-1\right)}\right) \\
 	& = \sqrt{\lambda_{i,t}}-\sqrt{\widehat{\lambda}_{i,t}} \\
 	& = \left|\sqrt{\lambda_{i,t}} - \sqrt{\widehat{\lambda}_{i,t}}\right|.
 \end{align}

This case study provides:

 \begin{equation}
 		\Wass(\pgenSDE_t,p_{t}) \leq \Wass(\pgenODE_t,p_{t}).
 	\end{equation}

\newpage

\section{Solution of the backward equations for a nonzero mean Gaussian distributions}
\label{appendix:solutions_with_non_zero_mean}

In the main paper, we assume a zero mean Gaussian for two reasons: first, to align with the machine learning framework, where data is typically normalized as a preprocessing; second, to simplify the notation of the equations. Nonetheless, all the results can be extended to the case of a nonzero mean so
let us provide the solutions to the backward equations to the case $\pdata = \mathcal{N}(\bmu,\bSigma)$.
% with $\bmu$ that can be nonzero.
In this case, $p_t(\x) = \mathcal{N}(e^{-B_t}\bmu,\bSigma_t)$ \cite{albergo_stochastic_interpolants_2023_arxiv} and
\begin{equation}
\label{eq:score_non_zero_mean}
    \nabla_{x} \log p_t(\x) = -\bSigma_t^{-1}(\x-e^{-B_t}\bmu), \quad 0< t \leq T.
\end{equation}

\subsection{Solution of the backward SDE}

The backward equation is given by
\begin{equation}
 d\y_t = \beta_{T-t}[\y_t - 2 \bSigma^{-1}_{T-t}(\y_t-e^{-B_{T-t}}\bmu)]dt + \sqrt{2\beta_{T-t}}d\w_t, \quad  0 \leq t \leq T.
\end{equation}

By considering $\z_t = \y_t - e^{-B_{T-t}}\bmu$,
\begin{align}
    d\z_t 
    & = d\y_t - \beta_{T-t} e^{-B_{T-t}}\bmu dt \\
    & = \beta_{T-t}[\y_t - 2 \bSigma^{-1}_{T-t}(\y_t-e^{-B_{T-t}}\bmu)]dt + \sqrt{2\beta_{T-t}}d\w_t- \beta_{T-t} e^{-B_{T-t}}\bmu dt \\
    & = \beta_{T-t}[\z_t - 2 \bSigma^{-1}_{T-t}\z_t]dt + \sqrt{2\beta_{T-t}}d\w_t
\end{align}
which is the backward equation for the zero mean distribution $\mathcal{N}(\zero,\bSigma)$ (Equation~\eqref{eq:sde_backward}). Consequently, by Proposition~\ref{prop:solution_backward_sde_gaussian_assumption}, 
\begin{equation}
    \z_t = e^{-(B_{T}-B_{T-t})}\bSigma_{T-t}\bSigma^{-1}_T\z_0 +\bxi_t, \quad 0 \leq t \leq T
\end{equation}
with $\bxi_t = e^{-C_{t}}\bSigma_{T-t}\int_0^t\bSigma^{-1}_{T-s} e^{C_{s}}\sqrt{2\beta_{T-s}}d\w_s$ 
%and $\Cov(\bxi_t) = \bSigma_{T-t} - e^{-2C_{ B_T - B_{T-t}}}\bSigma_{T-t}^2\bSigma^{-1}_{T}$ 
and
\begin{equation}
    \y_t = e^{-B_{T-t}}\bmu+ e^{-(B_{T}-B_{T-t})}\bSigma_{T-t}\bSigma^{-1}_T\left(\y_0-e^{-B_T}\mu\right) +\bxi_t, \quad 0 \leq t \leq T.
\end{equation}

\subsection{Solution of the probability flow ODE}

The reverse probability flow ODE is given by
\begin{equation}
 	d\y_t = \beta_{T-t}\left[\y_t -\bSigma^{-1}_{T-t}(\y_t-e^{-B_{T-t}}\bmu)\right]dt, \quad 0 \leq t < T
  \end{equation}

By considering $\z_t = \y_t - e^{-B_{T-t}}\bmu$,
\begin{align}
    d\z_t 
    & = d\y_t - \beta_{T-t} e^{-B_{T-t}}\bmu dt \\
    & = \beta_{T-t}\left[\y_t -\bSigma^{-1}_{T-t}(\y_t-e^{-B_{T-t}}\bmu)\right]dt- \beta_{T-t} e^{-B_{T-t}}\bmu dt \\
    & =\beta_{T-t}\left[\z_t -\bSigma^{-1}_{T-t}\z_t\right]dt
\end{align}
which is the reverse probability-flow for the zero mean distribution $\mathcal{N}(\zero,\bSigma)$ (Equation~\ref{eq:flow_reverse_ode}). Consequently, by Proposition~\ref{prop:solution_ode_Gaussian}, 
\begin{equation}
    \z_t =\bSigma_T^{-1/2}\bSigma_{T-t}^{1/2}\z_0, \quad 0 \leq t \leq T
\end{equation}
%with $\bxi_t = e^{-C_{t}}\bSigma_{T-t}\int_0^t\bSigma^{-1}_{T-s} e^{C_{s}}\sqrt{2\beta_{T-s}}d\w_s$ 
%and $\Cov(\bxi_t) = \bSigma_{T-t} - e^{-2C_{ B_T - B_{T-t}}}\bSigma_{T-t}^2\bSigma^{-1}_{T}$ 
and
\begin{equation}
    \y_t =e^{-B_{T-t}}\bmu+\bSigma_T^{-1/2}\bSigma_{T-t}^{1/2}(\y_0-e^{-B_T}\bmu), \quad 0 \leq t \leq T, \quad 0 \leq t \leq T.
\end{equation}

\newpage
 \section{Gaussian CIFAR-10 samples}

 \label{appendix:Gaussian_cifar10_samples}

 The Gaussian CIFAR-10 produces unstructured images. A grid of samples is presented in Figure~\ref{fig:samples_cifar10_Gaussian}. %{
 To sample from this Gaussian, the empirical covariance matrix of size $\R^{(3\times 32 \times 32)\times (3 \times 32 \times 32)}$is computed and then the SVD decomposition to extract a square root matrix (see source code).
 \begin{figure}[!ht]
 	\includegraphics[width = \linewidth]{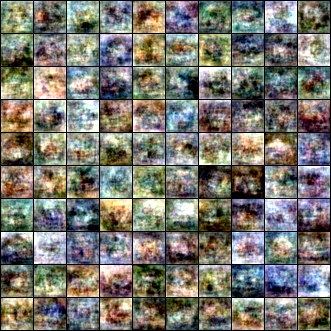}
 	\caption{\label{fig:samples_cifar10_Gaussian} \textbf{CIFAR-10 Gaussian samples}. Samples are generated from the Gaussian distribution fitting the CIFAR-10 dataset.}
 \end{figure}

\newpage
\section{Studied numerical schemes}

\label{appendix:table_numerical_schemes}

The studied numerical schemes are presented in Table~\ref{tab:discretization_schemes}.

\setlength{\tabcolsep}{6pt} % Default value: 6pt

\newlength{\schemewidth}
\setlength{\schemewidth}{12cm}

 \begin{table*}[ht]
 \scriptsize
	\centering
	\begin{tabular}{@{}lll@{}}
    && \hspace{\schemewidth} \\
	\toprule
\multirow{3}{*}{\rotatebox[origin=t]{90}{SDE schemes}} &\parbox{1cm}{Euler-Maruyama (EM)}  & 
\parbox{\schemewidth}{\begin{equation}
	\label{eq:Euler_Maruyama:SDE_backward}
    \begin{array}{l}
	\left\{
    \begin{array}{rl}
    \ySDEEM_{0} &\sim \N \\
	\ySDEEM_{k+1} & = \ySDEEM_{k} + \Delta_t \fSDE(t_k,\ySDEEM_{k}) + \sqrt{2\Delta_t \beta_{T-t_k}}\z_k, ~	\z_k \sim \N \text{\hspace{2.5cm}}
	\end{array}\right. \\
    \text{where}~
        \fSDE(t,\y) = \beta_{T-t}\y - 2\beta_{T-t}\bSigma_{T-t}^{-1}\y
    \end{array}
\end{equation}}
\\
\cmidrule(lr){2-3}%
&\parbox{1cm}{Exponential integrator (EI)} & 
\parbox{\schemewidth}{\begin{equation}
	\label{eq:EI_scheme:SDE}
	\begin{array}{l}
	\left\{%
    \begin{array}{rl}
        \ySDEEI_{0} &\sim \N \\
        \ySDEEI_{k+1} & = \ySDEEI_{k} + \gamma_{1,k}\left(\ySDEEI_{k} -2\bSigma^{-1}_{T-t_k}\ySDEEI_{k}\right) + \sqrt{2\gamma_{2,k}}z_k,
        ~\z_k \sim \N \text{\hspace{.7cm}}
    \end{array}\right.\quad  \quad \quad\quad \quad
    \\
    \text{where}~\gamma_{1,k} = \exp(B_{T-t_k}-B_{T-t_{k+1}})-1 ~\text{and}~
        \gamma_{2,k} = \frac{1}{2}(\exp(2 B_{T-t_k} - 2B_{T-t_{k+1}})-1)
    \end{array}
\end{equation}}
\\
\cmidrule(lr){2-3}%
&\parbox{1.8cm}{Denoising $\quad$ Diffusion $\qquad$ Probabilistic Model (DDPM)} & 
\parbox{\schemewidth}{\begin{equation}
	\label{eq:DDPM_scheme:SDE}
	\begin{array}{l}
	\left\{%
    \begin{array}{rl}
        \ySDEDDPM_{0} &\sim \N \\
         \ySDEDDPM_{k+1} & = \frac{1}{\sqrt{1-\beta_k^\text{DDPM}}}\left(\ySDEDDPM_{k}- \beta_k^\text{DDPM} \bSigma^{-1}_{T-t_k}\ySDEDDPM_{k}\right) + \sqrt{\beta^\text{DDPM}_k}z_k,~\z_k \sim \N \text{\hspace{.3cm}}
    \end{array}\right.
    \\
    \text{where}~\beta_k^\text{DDPM} = 2\Delta_t\beta(t_k)
    \end{array}
\end{equation}}\\
\midrule
\multirow{3}{*}{\rotatebox[origin=c]{90}{ODE schemes}} &\parbox{1cm}{Explicit Euler} &
\parbox{\schemewidth}{\begin{equation}
	\label{eq:Euler_explicite:flow_ODE}
    \begin{array}{l}
	\left\{
    \begin{array}{rl}
        \yODEEuler_{0} &\sim \N \\
        \yODEEuler_{k+1}  & = \yODEEuler_{k}+ \Delta_t \fODE(t_k,\yODEEuler_k) \text{\hspace{4.6cm}}
    \end{array}
\right. \\
\text{where}~\fODE(t,\y) = \beta_{T-t}\y - \beta_{T-t}\bSigma_{T-t}^{-1}\y
\end{array}
\quad \quad \quad \quad \thinspace
\end{equation}}\\
\cmidrule(lr){2-3}
&\parbox{1cm}{Heun's method} & 
\parbox{\schemewidth}{\begin{equation}
	\label{eq:Heun_scheme:flow_ODE}
    \begin{array}{l}
	\left\{
    \begin{array}{rl}
    \yODEHeun_{0} & \sim \N \\
        \yODEHeun_{k+1/2} & = \yODEHeun_{k}+ \Delta_t\fODE(t_k,\yODEHeun_{k}) \\
        \yODEHeun_{k+1} & = \yODEHeun_{k}+ \frac{\Delta_t}{2}\left(\fODE(t_k,\yODEHeun_{k})+\fODE(t_{k+1},\yODEHeun_{k+1/2})\right) \text{\hspace{2cm}}
    \end{array}
\right. \\
\text{where}~
        \fODE(t,\y) = \beta_{T-t}\y - \beta_{T-t}\bSigma_{T-t}^{-1}\y 
\end{array}
\quad \quad \quad \thinspace \thinspace
\end{equation}}\\
\cmidrule(lr){2-3}
&\parbox{.75cm}{Runge-Kutta 4 (RK4)} & 
\parbox{\schemewidth}{\begin{equation}
	\label{eq:RK4_scheme:flow_ODE}
    \begin{array}{l}
	\left\{
    \begin{array}{rl}
    \yODERK_{0} & \sim \N \\
    K_1 & =\fODE(t_k,\yODERK_k) \\
    K_2 & = \fODE(t_{k+1/2},\yODERK_{k}+\tfrac{\Delta_t}{2}K_1) \\
     K_3 & = \fODE(t_{k+1/2},\yODERK_{k}+\tfrac{\Delta_t}{2}K_2) \\
      K_4 & = \fODE(t_{k+1},\yODERK_{k}+\Delta_t K_3) \\
      \yODERK_{k+1} & = \yODERK_k + \frac{\Delta_t}{6}\left(K_1+2K_2+2K_3+K_4\right)
    \end{array}
\right.\\
\text{where}~
        \fODE(t,\y) = \beta_{T-t}\y - \beta_{T-t}\bSigma_{T-t}^{-1}\y, t_{k+1/2}=t_{k}+\tfrac{\Delta_t}{2}\text{\hspace{3.2cm}}
 \end{array}
 \quad \quad \quad \thinspace \thinspace
\end{equation}}\\
	\bottomrule
	\end{tabular}
	\caption{\footnotesize\label{tab:discretization_schemes} \textbf{Stochastic and deterministic discretization schemes}. EM, EI and DDPM disctretize the backward SDE of \Cref{eq:sde_backward_generative}, Euler, Heun and RK4 schemes discretize of the probability flow ODE of \Cref{eq:flow_ode_gen} with a regular time schedule $\left(t_k\right)_{0\leq k\leq N}$ with stepsize $\Delta_t = \tfrac{T}{N}$}.
\end{table*}

\newpage
\section{Computation of the 2-Wasserstein distances for numerical schemes}
%{
\label{appendix:computation_W2_errors}

The 2-Wasserstein errors can be computed by using \Cref{eq:W2_Gaussian_commute_diago} recalled here:

\begin{equation}
	\Wass(\mathcal{N}(\zero,\bSigma_1),\mathcal{N}(\zero,\bSigma_2))^2 =  \sum_{1 \leq i \leq d} (\sqrt{\lambda_{i,1}}-\sqrt{\lambda_{i,2}})^2.
 \tag{\ref{eq:W2_Gaussian_commute_diago}}
\end{equation}

for two centered Gaussians $\mathcal{N}(\zero,\bSigma_1)$ and $\mathcal{N}(\zero,\bSigma_2)$ such that $\bSigma_1,\bSigma_2$ are simultaneously diagonalizable with respective eigenvalues $\left(\lambda_{i,1}\right)_{1\leq i \leq d},\left(\lambda_{i,2}\right)_{1\leq i \leq d}$. We aim at computing these errors between the Gaussian process folowing $\left(p_t\right)_{0 \leq t \leq T}$ and the processes induced by the discretization schemes. Table~\ref{tab:discretization_schemes} shows that each discretization scheme leads to a discrete time Gaussian process whose covariance matrix is diagonalizable in the basis of $\bSigma$ when intialize them with either $\N$ (usual sampling) or $p_T$ (no initialization error). Let detail this point for the Euler-Maruyama (EM) scheme. Let denote %$\left(\lambda_i\right)_{1 \leq i \leq d}$ the eigenvalues of $\bSigma$ and 
$\left(\V_i\right)_{1 \leq i \leq d}$ a basis of eigenvectors of $\bSigma$ and its associated eigenvalues $\left(\lambda_i\right)_{1 \leq i \leq d}$. For $0 \leq t \leq T$, $\bSigma_t = e^{-2B_t} \bSigma + (1-e^{-2B_t})\I$ and for all $1 \leq i \leq d$,
\begin{equation}
    \bSigma_t \V_i = \left(e^{-2B_t} \lambda_i + (1-e^{-2B_t})\right)\V_i
\end{equation}
Consequently $\bSigma_t$ admits the same eigenvectors than $\bSigma$ with associated eigenvalues $\left(\lambda_{i,t}\right)_{1 \leq i \leq d} = \left(e^{-2B_t}\lambda_i+(1-e^{-2B_t})\right)_{1 \leq i \leq d}$. Let study the covariance matrix of the EM process. 
Let denote $\left(\bSigma_{k}^\text{$\Delta$,EM}\right)_{\substack{1 \leq i \leq d, 0 \leq k \leq N}}$ the covariance matrix of the Gaussian process generated by the EM scheme at each step and $\left(\lambda_{i,k}^\text{$\Delta$,EM}\right)_{\substack{1 \leq i \leq d, 0 \leq k \leq N}}$ its eigenvalues. First,
\begin{equation}
    \bSigma_{0}^\text{$\Delta$,EM} = \left\{
    \begin{array}{ll}
        \I & \mbox{if $\y_T$ is initialized at $\N$ }  \\
        \bSigma_T & \mbox{if $\y_T$ is initialized at $p_T$ }  \\
    \end{array}
\right..
\end{equation}
And consequently,
\begin{equation}
    \lambda_{i,0}^\text{$\Delta$,EM} = \left\{
    \begin{array}{ll}
        1 & \mbox{if $\y_T$ is initialized at $\N$ }  \\
        e^{-2B_T}\lambda_i+ (1-e^{-2B_T}) & \mbox{if $\y_T$ is initialized at $p_T$ }  \\
    \end{array}
\right. \quad 1\leq i\leq d.
\end{equation}

Then, by Table~\ref{tab:discretization_schemes},
\begin{align}
    \ySDEEM_{1}
    & = \left(\I+ \Delta_t \beta_{T-t_0}\left(\I-2\bSigma_{T-t_0}^{-1}\right)\right)\ySDEEM_{0}+ \sqrt{2\Delta_t \beta_{T-t_0}}\z_0, ~	\z_0 \sim \N
\end{align}
and
\begin{align}
    \bSigma_{1}^\text{$\Delta$,EM}
    & = \left(\I+ \Delta_t \beta_{T-t_0}\left(\I-2\bSigma_{T-t_0}^{-1}\right)\right) \bSigma_{0}^\text{$\Delta$,EM} \left(\I+ \Delta_t \beta_{T-t_0}\left(\I-2\bSigma_{T-t_0}^{-1}\right)\right)^T + 2\Delta_t \beta_{T-t_0}\I \\
    & = \left(\I+ \Delta_t \beta_{T-t_0}\left(\I-2\bSigma_{T-t_0}^{-1}\right)\right)^2\bSigma_{0}^\text{$\Delta$,EM} + 2\Delta_t \beta_{T-t_0}\I \text{ because $\bSigma$ and $\bSigma_{0}^\text{$\Delta$,EM}$ commute}.
\end{align}
Let $1 \leq i \leq d$,
\begin{align}
    \bSigma_{1}^\text{$\Delta$,EM}\V_i
    & = \left[\left(1+ \Delta_t \beta_{T-t_0}\left(\I-\frac{2}{\lambda_{i,T-t_0}}\right)\right)^2\lambda_{i,0}^\text{$\Delta$,EM} + 2\Delta_t \beta_{T-t_0}\right] \V_i
\end{align}
Consequently, $(\V_i)_{1 \leq i \leq d}$ is also a basis of eigenvectors of $\bSigma_{1}^\text{$\Delta$,EM}$ and
\begin{equation}
    \lambda_{i,1}^\text{$\Delta$,EM} = \left(1+ \Delta_t \beta_{T-t_0}\left(\I-\frac{2}{\lambda_{i,T-t_0}}\right)\right)^2\lambda_{i,0}^\text{$\Delta$,EM} + 2\Delta_t \beta_{T-t_0}, \quad 1 \leq i \leq d.
\end{equation}
Thus, we can obtain the eigenvalues $\left(\lambda_{i,k}^\text{$\Delta$,EM}\right)_{\substack{1 \leq i \leq d, 0 \leq k \leq N}}$ at each time and plot at each time
\begin{equation}
    \sqrt{\sum_{1\leq i \leq d} \left(\sqrt{\lambda_{i,T-t_k}}-\sqrt{\lambda_{i,k}^\text{$\Delta$,EM}}\right)}, \quad 1 \leq k \leq N
\end{equation}
as done in Figure~\ref{fig:cifar_measures}. These computations can be led for the different schemes, as presented in Table~\ref{tab:schemes_eigenvalues}.

\begin{table}[H]
    \centering
     \tiny
    \begin{tabular}{cll}
    \toprule 
        \multirow{2}{*}{\rotatebox[origin=t]{90}{\tiny SDE schemes \hspace{.2cm}}} & EM & $\lambda_{i,k+1}^{\Delta,\text{EM}} = \left(1+\Delta_t \beta_{T-t_k}\left(1-\tfrac{2}{\lambda_{i,T-t_k}}\right)\right)^2\lambda_{i,k}^{\Delta,\text{EM}} + 2\Delta_t \beta_{T-t_k}, \quad 1 \leq i \leq d, 0 \leq k \leq N-1$ \\
        \cmidrule{2-3}
          & EI & $\lambda_{i,k+1}^{\Delta,\text{EI}} = \left(1+\gamma_{1,k}\left(1-\tfrac{2}{\lambda_{i,T-t_k}}\right)\right)^2\lambda_{i,k}^{\Delta,\text{EI}}+ 2\gamma_{2,k}\quad 1 \leq i \leq d, 0 \leq k \leq N-1$ \\
          \cmidrule{2-3}
          & DDPM & $\lambda_{i,k+1}^{\Delta,\text{DDPM}} = \frac{1}{1-\beta_k^\text{DDPM}} \left(1- \tfrac{\beta_k^\text{DDPM}}{\lambda_{i,T-t_k}}\right)^2\lambda_{i,k}^{\Delta,\text{DDPM}}+ \beta^\text{DDPM}_k\quad 1 \leq i \leq d, 0 \leq k \leq N-1$ \\
          \midrule
          \multirow{2}{*}{\rotatebox[origin=t]{90}{\tiny ODE schemes \hspace{1cm}}} & Euler & 
          $\lambda_{i,k+1}^{\Delta,\text{Euler}} =\left(1+\Delta_t\beta_{T-t_k}\left(1-\tfrac{1}{\lambda_{i,T-t_k}}\right)\right)^2\lambda_i^{\text{Euler},k}\quad 1 \leq i \leq d, 0 \leq k \leq N-1$
          \\
          \cmidrule{2-3}
          & Heun &  $\lambda_{i,k+1}^{\Delta,\text{Heun}}  = \left(1+ \frac{\Delta_t}{2} \beta_{T-t_k}\left(1-\tfrac{1}{\lambda_{i,T-t_k}}\right)+\frac{\Delta_t}{2} \beta_{T-t_{k+1}}\left(1-\tfrac{1}{\lambda_{i,T-t_{k+1}}}\right)\left( 1+ \Delta_t \beta_{T-t_k}\left(1-\tfrac{1}{\lambda_{i,T-t_k}}\right)\right)\right)\lambda_{i,k}^{\Delta,\text{Heun}}$\\
          & & $\quad 1 \leq i \leq d, 0 \leq k \leq N-1$\\
           \cmidrule{2-3}
          & RK4 &  
          $\begin{array}{ll}
               C_1 & = \beta_{T-t_k}\left(1-\tfrac{1}{\lambda_{i,T-t_k}}\right) \\
               C_2 & = \beta_{T-t_{k+1/2}}\left(1-\tfrac{1}{\lambda_{i,T-t_{k+1/2}}}\right)\left(1+\tfrac{\Delta_t}{2}C_1\right) \\
               C_3 & = \beta_{T-t_{k+1/2}}\left(1-\tfrac{1}{\lambda_{i,T-t_{k+1/2}}}\right)\left(1+\tfrac{\Delta_t}{2}C_2\right) \\
               C_4 & = \beta_{T-t_{k+1}}\left(1-\tfrac{1}{\lambda_{i,T-t_{k+1}}}\right)\left(1+\Delta_t C_3\right) \\
               \lambda_{i,k+1}^{\Delta,\text{RK4}}
               & = \left(1+\tfrac{\Delta_t}{6}(C_1+2C_2+2C_3+C_4)\right)^2\lambda_{i,k}^{\Delta,\text{RK4}} \\
               & \quad 1 \leq i \leq d, 0 \leq k \leq N-1
          \end{array}$ \\
          \bottomrule
    \end{tabular}
    \caption{Recursive form of the eigenvalues of the covariance matrix associated with the Gaussian process generated by the different schemes for a regular time schedule $\left(t_k\right)_{0\leq k\leq N}$ with steps $\Delta_t = \tfrac{T}{N}$.}
    \label{tab:schemes_eigenvalues}
\end{table}

\newpage

\section{Ablation study of Gaussian CIFAR-10}

\label{appendix:ablation_study_cifar10}

The complete ablation study of Gaussian CIFAR-10 is presented in Table~\ref{tab:ablation_study_cifar10}.

\setlength{\tabcolsep}{10pt} % Default value: 6pt
\begin{table}[H]
\centering
\footnotesize
\begin{tabular}{llllllllllll} 
\toprule
& &\multicolumn{2}{c}{Continuous}
& \multicolumn{2}{c}{NFE = $50$}
& \multicolumn{2}{c}{NFE = $250$}
& \multicolumn{2}{c}{NFE = $500$}
& \multicolumn{2}{c}{NFE = $1000$}
\\
\cmidrule(lr){3-4}
\cmidrule(lr){5-6}
\cmidrule(lr){7-8}
\cmidrule(lr){9-10}
\cmidrule(lr){11-12}
& &$p_T$ & $\mathcal{N}_0$ & 
$p_T$ & $\mathcal{N}_0$ &
$p_T$ & $\mathcal{N}_0$ &
$p_T$ & $\mathcal{N}_0$ &
$p_T$ & $\mathcal{N}_0$ \\
\midrule
\parbox[t]{2mm}{\multirow{4}{*}{\rotatebox[origin=c]{90}{EM}}}
& \multicolumn{1}{|l}{$\varepsilon = 0$}
 & 0
 & 6.7E-04
 & 4.78
 & 4.78
 & 0.65
 & 0.66
 & 0.32
 & 0.32
 & 0.16
 & 0.16
\\
& \multicolumn{1}{|l}{$\varepsilon = 10^{-5}$}
 & 4.1E-03
 & 4.2E-03
 & 4.77
 & 4.77
 & 0.66
 & 0.66
 & 0.32
 & 0.32
 & 0.16
 & 0.16
\\
& \multicolumn{1}{|l}{$\varepsilon = 10^{-4}$}
 & 0.03
 & 0.03
 & 4.76
 & 4.76
 & 0.66
 & 0.66
 & 0.32
 & 0.32
 & 0.17
 & 0.17
\\
& \multicolumn{1}{|l}{$\varepsilon = 10^{-3}$}
 & 0.18
 & 0.18
 & 4.68
 & 4.68
 & 0.70
 & 0.70
 & 0.40
 & 0.40
 & 0.27
 & 0.27
\\
&&&&&&&&&&&\\
\parbox[t]{2mm}{\multirow{4}{*}{\rotatebox[origin=c]{90}{EI}}}
& \multicolumn{1}{|l}{$\varepsilon = 0$}
&&
 & 2.81
 & 2.81
 & 0.57
 & 0.57
 & 0.30
 & 0.30
 & 0.16
 & 0.16
\\
& \multicolumn{1}{|l}{$\varepsilon = 10^{-5}$}
&&
 & 2.81
 & 2.81
 & 0.57
 & 0.57
 & 0.30
 & 0.30
 & 0.16
 & 0.16
\\
& \multicolumn{1}{|l}{$\varepsilon = 10^{-4}$}
&&
 & 2.82
 & 2.82
 & 0.58
 & 0.58
 & 0.31
 & 0.31
 & 0.17
 & 0.17
\\
& \multicolumn{1}{|l}{$\varepsilon = 10^{-3}$}
&&
 & 2.91
 & 2.91
 & 0.67
 & 0.67
 & 0.41
 & 0.41
 & 0.29
 & 0.29
\\
&&&&&&&&&&&\\
\parbox[t]{2mm}{\multirow{4}{*}{\rotatebox[origin=c]{90}{DDPM}}}
& \multicolumn{1}{|l}{$\varepsilon = 0$}
&&
 & 6.82
 & 6.82
 & 0.97
 & 0.97
 & 0.47
 & 0.47
 & 0.23
 & 0.23
\\
& \multicolumn{1}{|l}{$\varepsilon = 10^{-5}$}
&&
 & 6.82
 & 6.82
 & 0.97
 & 0.97
 & 0.47
 & 0.47
 & 0.24
 & 0.23
\\
& \multicolumn{1}{|l}{$\varepsilon = 10^{-4}$}
&&
 & 6.81
 & 6.81
 & 0.98
 & 0.98
 & 0.47
 & 0.47
 & 0.24
 & 0.24
\\
& \multicolumn{1}{|l}{$\varepsilon = 10^{-3}$}
&&
 & 6.74
 & 6.74
 & 1.00
 & 1.00
 & 0.53
 & 0.53
 & 0.32
 & 0.32
\\
&&&&&&&&&&&\\
\parbox[t]{2mm}{\multirow{4}{*}{\rotatebox[origin=c]{90}{Euler}}}
& \multicolumn{1}{|l}{$\varepsilon = 0$}
 & 0
 & 0.07
 & 1.72
 & 1.78
 & 0.38
 & 0.44
 & 0.20
 & 0.26
 & 0.10
 & 0.17
\\
& \multicolumn{1}{|l}{$\varepsilon = 10^{-5}$}
 & 4.1E-03
 & 0.07
 & 1.72
 & 1.78
 & 0.38
 & 0.44
 & 0.20
 & 0.26
 & 0.10
 & 0.17
\\
& \multicolumn{1}{|l}{$\varepsilon = 10^{-4}$}
 & 0.03
 & 0.08
 & 1.72
 & 1.78
 & 0.38
 & 0.44
 & 0.20
 & 0.26
 & 0.11
 & 0.17
\\
& \multicolumn{1}{|l}{$\varepsilon = 10^{-3}$}
 & 0.18
 & 0.19
 & 1.73
 & 1.79
 & 0.42
 & 0.48
 & 0.27
 & 0.32
 & 0.21
 & 0.25
\\
&&&&&&&&&&&\\
\parbox[t]{2mm}{\multirow{4}{*}{\rotatebox[origin=c]{90}{Heun}}}
& \multicolumn{1}{|l}{$\varepsilon = 0$}
&&
 & -
 & -
 & -
 & -
 & -
 & -
 & -
 & -
\\
& \multicolumn{1}{|l}{$\varepsilon = 10^{-5}$}
&&
 & 23.42
 & 23.42
 & 2.86
 & 2.87
 & 1.05
 & 1.06
 & 0.37
 & 0.38
\\
& \multicolumn{1}{|l}{$\varepsilon = 10^{-4}$}
&&
 & 4.68
 & 4.68
 & 0.43
 & 0.44
 & 0.12
 & 0.14
 & 0.03
 & 0.08
\\
& \multicolumn{1}{|l}{$\varepsilon = 10^{-3}$}
&&
 & 0.58
 & 0.59
 & 0.13
 & 0.15
 & 0.16
 & 0.18
 & 0.17
 & 0.19
\\
&&&&&&&&&&&\\
\parbox[t]{2mm}{\multirow{4}{*}{\rotatebox[origin=c]{90}{RK4}}}
& \multicolumn{1}{|l}{$\varepsilon = 0$}
&&
 & -
 & -
 & -
 & -
 & -
 & -
 & -
 & -
\\
& \multicolumn{1}{|l}{$\varepsilon = 10^{-5}$}
&&
 & 67.32
 & 67.32
 & 5.49
 & 5.49
 & 2.11
 & 2.11
 & 0.77
 & 0.77
\\
& \multicolumn{1}{|l}{$\varepsilon = 10^{-4}$}
&&
 & 14.35
 & 14.35
 & 0.93
 & 0.93
 & 0.29
 & 0.29
 & 0.07
 & 0.10
\\
& \multicolumn{1}{|l}{$\varepsilon = 10^{-3}$}
&&
 & 2.60
 & 2.60
 & 0.09
 & 0.12
 & 0.15
 & 0.17
 & 0.17
 & 0.19
\\
\bottomrule
\end{tabular}
\caption{\small \label{tab:ablation_study_cifar10} \textbf{Ablation study of Wasserstein errors for the CIFAR-10 Gaussian.} For a given discretization scheme, the table presents the Wasserstein distance associated with the truncation error for different values of $\varepsilon$. The columns $p_T$ and $\N$ show the influence of the initialization error. The continuous column corresponds to the continuous SDE or ODE linked with the scheme (identical values for EM, EI, DDPM and Euler, Heun, RK4) and a given number of score function evaluation NFE. }
\end{table}

\newpage

\section{Theoretical Wasserstein distance for the ADSN model}

\label{appendix:ADSN_error}

 As done for the Gaussian CIFAR-10, the Wasserstein errors can be computed for the ADSN model as shown in Figure~\ref{fig:ADSN_measures} and Table~\ref{tab:ablation_study_ADSN}.

\newlength{\adsnerrorwidth}
\setlength{\adsnerrorwidth}{0.5\linewidth}

\setlength{\tabcolsep}{1pt} % Default value: 6pt
 \begin{figure}[h]
 \small
\begin{tabular}{cc}
    \includegraphics[width = \adsnerrorwidth]{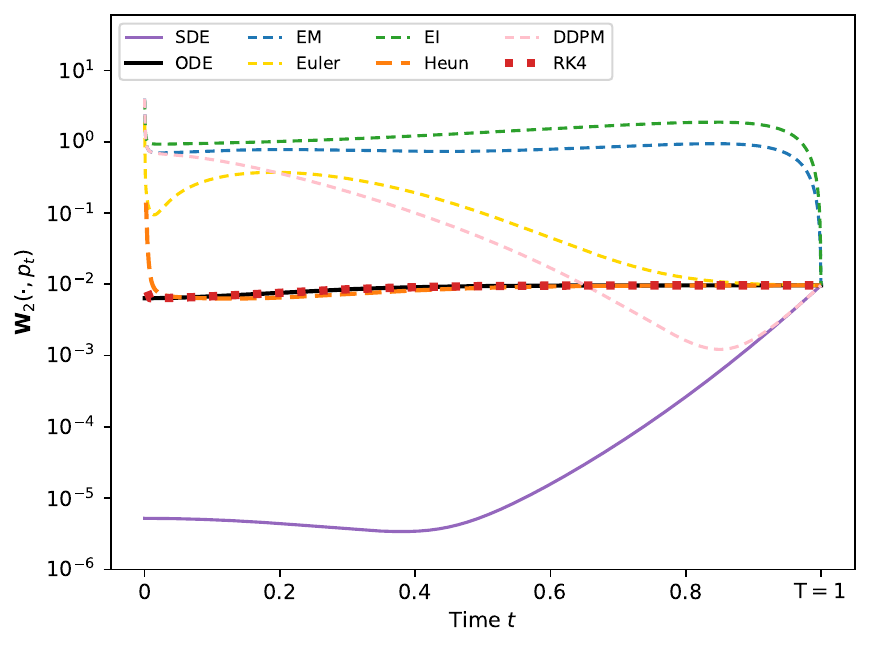} &
    \includegraphics[width = \adsnerrorwidth]{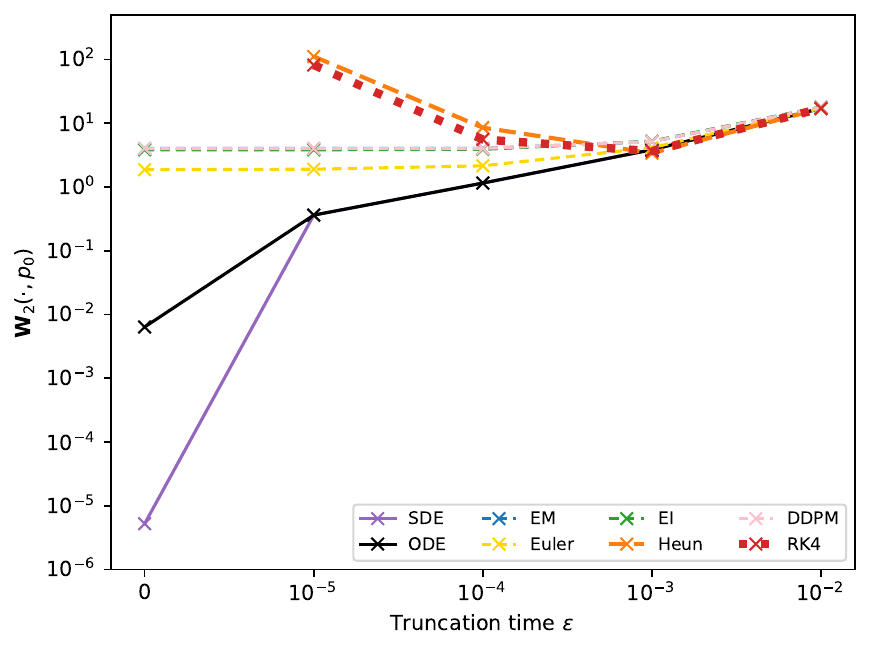} \\ (a) Initialization error along the integration time. &
    (b) Truncation error for different truncation time $\varepsilon$
\end{tabular}
 	\caption{\small \label{fig:ADSN_measures} \textbf{Wasserstein errors for the diffusion models associated with the Gaussian microtextures.} Left: Evolution of the Wasserstein distance between $p_t$ and the distributions associated with the continuous SDE, the continuous flow ODE and four discrete sampling schemes with standard $\N$ initialization, either stochastic (Euler-Maruyama (EM), Exponential Integrator (EI) and Denoising Diffusion Probabilistic Model (DDPM) ) or deterministic (Euler, Heun and Runge-Kutta 4 (RK4)).
While the continuous SDE is less sensible than the continuous ODE (as proved by Proposition~\ref{prop:marginal_generatives}), the initialization error impacts all discrete schemes.
Heun's method has the lowest error and is very close to the theoretical ODE, except for the last step that is usually discarded when using time truncation.
Right: Wasserstein errors due to time truncation for various truncation times $\varepsilon$. Heun's scheme is not defined without truncation time due to the zero eigenvalue.
Interestingly, for the standard practice truncation time $\varepsilon = 10^{-3}$, all numerical schemes have a comparable error close to their continuous counterparts.}

\end{figure}

\setlength{\tabcolsep}{5pt} % Default value: 6pt
\begin{table}[h]
\centering
\footnotesize
\begin{tabular}{llllllllllll} 
\toprule
& &\multicolumn{2}{c}{Continuous}
& \multicolumn{2}{c}{NFE = $50$}
& \multicolumn{2}{c}{NFE = $250$}
& \multicolumn{2}{c}{NFE = $500$}
& \multicolumn{2}{c}{NFE = $1000$}
\\
\cmidrule(lr){3-4}
\cmidrule(lr){5-6}
\cmidrule(lr){7-8}
\cmidrule(lr){9-10}
\cmidrule(lr){11-12}
& &$p_T$ & $\mathcal{N}_0$ & 
$p_T$ & $\mathcal{N}_0$ &
$p_T$ & $\mathcal{N}_0$ &
$p_T$ & $\mathcal{N}_0$ &
$p_T$ & $\mathcal{N}_0$ \\
\midrule
\parbox[t]{2mm}{\multirow{4}{*}{\rotatebox[origin=c]{90}{EM}}}
& \multicolumn{1}{|l}{$\varepsilon = 0$}
 & 0
 & 5.2E-06
 & 53.37
 & 53.37
 & 10.58
 & 10.58
 & 6.27
 & 6.27
 & 4.02
 & 4.02
\\
& \multicolumn{1}{|l}{$\varepsilon = 10^{-5}$}
 & 0.36
 & 0.36
 & 53.35
 & 53.35
 & 10.57
 & 10.57
 & 6.26
 & 6.26
 & 4.02
 & 4.02
\\
& \multicolumn{1}{|l}{$\varepsilon = 10^{-4}$}
 & 1.15
 & 1.15
 & 53.21
 & 53.21
 & 10.53
 & 10.53
 & 6.25
 & 6.25
 & 4.03
 & 4.03
\\
& \multicolumn{1}{|l}{$\varepsilon = 10^{-3}$}
 & 3.84
 & 3.84
 & 51.92
 & 51.92
 & 10.55
 & 10.55
 & 6.80
 & 6.80
 & 5.16
 & 5.16
\\
&&&&&&&&&&&\\
\parbox[t]{2mm}{\multirow{4}{*}{\rotatebox[origin=c]{90}{EI}}}
& \multicolumn{1}{|l}{$\varepsilon = 0$}
&&
 & 30.91
 & 30.91
 & 8.85
 & 8.85
 & 5.71
 & 5.71
 & 3.84
 & 3.84
\\
& \multicolumn{1}{|l}{$\varepsilon = 10^{-5}$}
&&
 & 30.92
 & 30.92
 & 8.85
 & 8.85
 & 5.72
 & 5.72
 & 3.84
 & 3.84
\\
& \multicolumn{1}{|l}{$\varepsilon = 10^{-4}$}
&&
 & 31.01
 & 31.01
 & 8.92
 & 8.92
 & 5.78
 & 5.78
 & 3.90
 & 3.90
\\
& \multicolumn{1}{|l}{$\varepsilon = 10^{-3}$}
&&
 & 31.94
 & 31.94
 & 9.74
 & 9.74
 & 6.76
 & 6.76
 & 5.24
 & 5.24
\\
&&&&&&&&&&&\\
\parbox[t]{2mm}{\multirow{4}{*}{\rotatebox[origin=c]{90}{DDPM}}}
& \multicolumn{1}{|l}{$\varepsilon = 0$}
&&
 & 53.66
 & 53.66
 & 10.58
 & 10.58
 & 6.27
 & 6.27
 & 4.02
 & 4.02
\\
& \multicolumn{1}{|l}{$\varepsilon = 10^{-5}$}
&&
 & 53.64
 & 53.64
 & 10.58
 & 10.58
 & 6.26
 & 6.26
 & 4.02
 & 4.02
\\
& \multicolumn{1}{|l}{$\varepsilon = 10^{-4}$}
&&
 & 53.50
 & 53.50
 & 10.54
 & 10.54
 & 6.25
 & 6.25
 & 4.03
 & 4.03
\\
& \multicolumn{1}{|l}{$\varepsilon = 10^{-3}$}
&&
 & 52.19
 & 52.19
 & 10.56
 & 10.56
 & 6.80
 & 6.80
 & 5.16
 & 5.16
\\
&&&&&&&&&&&\\
\parbox[t]{2mm}{\multirow{4}{*}{\rotatebox[origin=c]{90}{Euler}}}
& \multicolumn{1}{|l}{$\varepsilon = 0$}
 & 0
 & 6.4E-03
 & 5.69
 & 5.70
 & 3.27
 & 3.27
 & 2.50
 & 2.51
 & 1.87
 & 1.87
\\
& \multicolumn{1}{|l}{$\varepsilon = 10^{-5}$}
 & 0.36
 & 0.36
 & 5.70
 & 5.71
 & 3.28
 & 3.28
 & 2.53
 & 2.53
 & 1.90
 & 1.90
\\
& \multicolumn{1}{|l}{$\varepsilon = 10^{-4}$}
 & 1.15
 & 1.15
 & 5.80
 & 5.80
 & 3.43
 & 3.43
 & 2.72
 & 2.72
 & 2.15
 & 2.15
\\
& \multicolumn{1}{|l}{$\varepsilon = 10^{-3}$}
 & 3.84
 & 3.84
 & 6.79
 & 6.79
 & 4.85
 & 4.85
 & 4.41
 & 4.41
 & 4.14
 & 4.14
\\
&&&&&&&&&&&\\
\parbox[t]{2mm}{\multirow{4}{*}{\rotatebox[origin=c]{90}{Heun}}}
& \multicolumn{1}{|l}{$\varepsilon = 0$}
&&
 & -
 & -
 & -
 & -
 & -
 & -
 & -
 & -
\\
& \multicolumn{1}{|l}{$\varepsilon = 10^{-5}$}
&&
 & 2.4E+03
 & 2.4E+03
 & 3.0E+02
 & 3.0E+02
 & 1.1E+02
 & 1.1E+02
 & 40.00
 & 40.00
\\
& \multicolumn{1}{|l}{$\varepsilon = 10^{-4}$}
&&
 & 2.3E+02
 & 2.3E+02
 & 26.34
 & 26.34
 & 8.54
 & 8.54
 & 2.01
 & 2.01
\\
& \multicolumn{1}{|l}{$\varepsilon = 10^{-3}$}
&&
 & 15.42
 & 15.42
 & 2.25
 & 2.25
 & 3.40
 & 3.40
 & 3.73
 & 3.73
\\
&&&&&&&&&&&\\
\parbox[t]{2mm}{\multirow{4}{*}{\rotatebox[origin=c]{90}{RK4}}}
& \multicolumn{1}{|l}{$\varepsilon = 0$}
&&
 & -
 & -
 & -
 & -
 & -
 & -
 & -
 & -
\\
& \multicolumn{1}{|l}{$\varepsilon = 10^{-5}$}
&&
 & 6.9E+03
 & 6.9E+03
 & 5.7E+02
 & 5.7E+02
 & 2.2E+02
 & 2.2E+02
 & 82.03
 & 82.03
\\
& \multicolumn{1}{|l}{$\varepsilon = 10^{-4}$}
&&
 & 6.8E+02
 & 6.8E+02
 & 52.25
 & 52.25
 & 18.61
 & 18.61
 & 5.57
 & 5.57
\\
& \multicolumn{1}{|l}{$\varepsilon = 10^{-3}$}
&&
 & 59.79
 & 59.79
 & 0.62
 & 0.62
 & 2.94
 & 2.94
 & 3.66
 & 3.66
\\
\bottomrule
\end{tabular}
\caption{\small \label{tab:ablation_study_ADSN} \textbf{Ablation study of Wasserstein errors for the Gaussian microtextures.} For a given discretization scheme, the table presents the Wasserstein distance associated with the truncation error for different values of $\varepsilon$. The columns $p_T$ and $\N$ show the influence of the initialization error. The continuous column corresponds to the continuous SDE or ODE linked with the scheme (identical values for EM, EI, DDPM and Euler, Heun, RK4). Note that the Heun scheme is not defined without truncation time due to the zero eigenvalue.}
\end{table}

\newpage

\section{Study of the covariance matrix of the ADSN distribution}

\subsection{Reminders on the Discrete Fourier Transform (DFT)}

\label{appendix:reminder_dft}
For a given image $\V \in \R^{3\times \OMN}$, we define the DFT of $\V$, $\widehat{\V} \in \R^{3 \times \OMN}$ such that for $1 \leq c \leq 3$, $\xi \in \OMN$

\begin{equation}
	\widehat{\V}_{c,\xi} = \sum_{x \in \OMN} \V_{c,x} \exp({-\frac{2i\pi x_1 \xi_1}{M}})\exp(-\frac{2i\pi x_2 \xi_2}{N}), \quad i^2 = -1
\end{equation}

where $\widehat{\V}_{c,\xi}$ is the value of $\widehat{\V}$ at coordinate $\xi$ of the k-th channel of $\widehat{\V}$. For $\U \in \R^{3\OMN}$, by defining $\U \star \V$ the periodic convolution such that for $1 \leq c \leq 3, x\in \R^{\OMN}$:

\begin{equation}
	\left(\U\star \V \right)_{c,x} = \sum_{y \in \OMN} \U_{c,x-y}\V_{c,y}
\end{equation}

we have:

\begin{equation}
	\widehat{\U\star \V} = \widehat{\U}\odot \widehat{\V},
\end{equation}

where $\odot$ is the componentwise product.
By denoting $\overleftarrow{\V}$ the image which is reversing the arrangement of elements in vector $\V$ ,
\begin{equation}
    \widehat{\overleftarrow{v}}  = \overline{\widehat{v}}
\end{equation}
where $\overline{\widehat{v}}$ is the component-wise conjugate of $\widehat{v}$.

\subsection{Eigenvectors of the covariance matrix of the ADSN distribution}
\label{appendix:covariance_ADSN_eigen}

Let $\U \in \R^{3\times \OMN}$ and its associated texton $\bt \in \R^{3\times \OMN}$. The distribution $\ADSN(\U)$ is the Gaussian distribution of $\X = \bt \star \w$ such that:
\begin{equation}
	\label{eq:X_values_ADSN_samples}
	\X_i = \bt_i \star \w \in \R^{\OMN}, 1 \leq i \leq 3, \w \sim \N
\end{equation}

Consequently, denoting $\bSigma$ the covariance of $\ADSN(\U)$, for $\V \in \R^{3\OMN}$, $\widehat{\bSigma\V}$ the Fourier transform of $\bSigma\V$ and $\widehat{\bSigma\V}_i$ its $i$th channel for $1 \leq i \leq 3$,

\begin{equation}
	\label{eq:ADSN_covariance_action}
	\widehat{\bSigma\V}_i 
    %=  \widehat{\bt}_i \overline{\widehat{\bt}_1} \widehat{\V}_1+\widehat{\bt}_i \overline{\widehat{\bt}_2} \widehat{\V }_2+\widehat{\bt}_i \overline{\widehat{\bt}_3} \widehat{\V}_3 
    = \widehat{\bt}_i \odot \left(\overline{\widehat{\bt}_1} \odot\widehat{\V}_1+\overline{\widehat{\bt}_2} \odot\widehat{\V }_2+ \overline{\widehat{\bt}_3} \odot\widehat{\V}_3\right)
\end{equation}

This equation proves that the kernel of $\bSigma$ contains the kernel of $\V \in \R^{3 \times \OMN} \mapsto \overline{\widehat{\bt}_1} \widehat{\V}_1+\overline{\widehat{\bt}_2} \widehat{\V }_2+ \overline{\widehat{\bt}_3} \widehat{\V}_3 \in \R^{\OMN}$ which has a dimension greater than $2MN$. Consequently, $0$ is eigenvalue of $\bSigma$ with multiplicity greater than $2MN$. Furthermore, for $\xi \in \OMN$, denoting $\U^{1,\xi}$ such that:
\begin{equation}
	\label{eq:def_v_xi}
	\widehat{\U}_i^{1,\xi}(\omega) = \mathbf{1}_{\omega = \xi} \widehat{\bt}_i(\omega), 1 \leq i \leq 3, \omega \in \OMN
\end{equation}

we have,
\begin{equation}
	\label{eq:v_vi_is_eigen_vector}
	\bSigma\U^{1,\xi} = (|\widehat{\bt}_1(\xi)|^2+|\widehat{\bt}_2(\xi)|^2+|\widehat{\bt}_3(\xi)|^2)\U^{1,\xi}.
\end{equation}

Furthermore, the family $\left(\U^{1,\xi}\right)_{\xi \in \OMN}$ is orthogonal. Thus, the eigenvalues of $\bSigma$ are $\left(|\widehat{\bt}_1(\xi)|^2+|\widehat{\bt}_2(\xi)|^2+|\widehat{\bt}_3(\xi)|^2\right)_{\xi \in \OMN}$ and $0$ with multiplicity $2MN$.

For $\xi \in \OMN$, we denote $\U^{2,\xi}$,$\U^{3,\xi}$ such that for $\omega \in \OMN$:

\begin{align}
	\label{eq:def_v_ker_xi}
	& \left\{
    \begin{array}{ll}
        \widehat{\U}_1^{2,\xi}(\omega) & = -\mathbf{1}_{\omega = \xi}\overline{\widehat{\bt}}_3(\omega)   \\
        \widehat{\U}_2^{2,\xi}(\omega) & = 0  \\
        \widehat{\U}_3^{2,\xi}(\omega) & = \mathbf{1}_{\omega = \xi}\overline{\widehat{\bt}}_1(\omega)  \\
    \end{array}
\right. \\
& \left\{
    \begin{array}{ll}
        \widehat{\U}_1^{3,\xi}(\omega) & = 0   \\
        \widehat{\U}_2^{3,\xi}(\omega) & = -\mathbf{1}_{\omega = \xi}\overline{\widehat{\bt}}_3(\omega)  \\
        \widehat{\U}_3^{3,\xi}(\omega) & = \mathbf{1}_{\omega = \xi}\overline{\widehat{\bt}}_2(\omega) \\
    \end{array}
\right.
\end{align}

We have

\begin{align}
	\bSigma\U^{2,\xi} & = 0.\U^{2,\xi} \\
	\bSigma\U^{3,\xi} & = 0.\U^{3,\xi}.
\end{align}

Then, applying the orthonomalization of Gram-Schmidt on each tuple $(\U^{1,\xi},\U^{2,\xi},\U^{3,\xi})_{\xi \in \OMN }$, we obtain an orthonormal basis in the Fourier domain $(\widehat{\V}^{1,\xi},\widehat{\V}^{2,\xi},\widehat{\V}^{3,\xi})_{\xi \in \OMN }$ of eigenvectors of $\bSigma$. More precisely, for $\xi_1,\xi_2 \in \OMN$, $1 \leq j_1, j_2 \leq 3$,

\begin{align}
\left(\overline{\widehat{\V}}^{j_1,\xi_1}\right)^T\widehat{\V}^{j_2,\xi_2}
	& = \sum_{\substack{x_1 \in \OMN \\ x_2 \in \OMN}}\overline{\widehat{\V}}^{j_1,\xi_1}_{x_1}\widehat{\V}^{j_2,\xi_2}_{x_2} \\
	&= \mathbf{1}_{\substack{j_1 = j_2 \\ \xi_1 = \xi_2}}
\end{align}

which is applying the square root of $\bSigma$ to the white Gaussian noise $\w$. Furthermore, we can ensure that for $\xi \neq \omega \in \OMN, 1 \leq j \leq 3$, $\widehat{\V}^{j,\xi}(\omega)=0$ such that only the frequency $\xi$ is active for $\widehat{\V}^{j}$. Consequently, for $\w \in \R^{3\OMN}$,

\begin{equation}
	\label{eq:v_inner}
	\overline{\widehat{\w}}^T {\V}^{j,\xi} = \sum_{1 \leq i \leq 3} \overline{\widehat{\w}}_{i}(\xi) \widehat{\V}^{j,\xi}_{i}(\xi).
\end{equation}

In particular,

\begin{equation}
	\label{eq:v_norm}
	\left(\widehat{\overline{\V}}^{j,\xi}\right)^T \widehat{\overline{\V}}^{j,\xi} = \|\widehat{\V}^{j,\xi}\|^2  = \sum_{1 \leq i \leq 3} \left|{\V}^{j,\xi}_{i}(\xi)\right|^2 = 1.
\end{equation}

\subsection{Computation of the empirical Wasserstein error in the ADSN covariance diagonalization basis}
\label{appendix:covariance_ADSN_W2}

Let consider a Gaussian distribution $\mathcal{N}(\zero,\G)$ such that there exists $(\lambda_1^{\xi},\lambda_2^{\xi},\lambda_3^{\xi})_{\xi \in \OMN }$ such that for all $\xi \in \OMN$,
\begin{equation}
	\G\V^{j,\xi}  = \lambda_j^{\xi}\V^{j,\xi}, \quad 1 \leq j \leq 3.
\end{equation}
Let $\w\sim \N \in \R^{3\OMN}$,  $(\V^{1,\xi},\V^{2,\xi},\V^{3,\xi})_{\xi \in \OMN }$ is an orthonormal basis in the Fourier domain such that:
\begin{align}
	\widehat{\w}
	& = \sum_{\xi \in \OMN}\left(\left[\overline{\widehat{\w}}^T\widehat{\V}^{1,\xi}\right]\widehat{\V}^{1,\xi}+\left[\overline{\widehat{\w}}^T\widehat{\V}^{2,\xi}\right]\widehat{\V}^{2,\xi}+\left[\overline{\widehat{\w}}^T\widehat{\V}^{3,\xi}\right]\widehat{\V}^{3,\xi}\right) \\
\end{align}
A sample drawn from $\mathcal{N}(\zero,\G)$ has the same distribution as $\Y$ given by

\begin{equation}
	\widehat{\Y} = \sum_{\xi \in \OMN}\sqrt{\lambda_1^{\xi}} \left[\overline{\widehat{\w}}^T\widehat{\V}^{1,\xi}\right]\widehat{\V}^{1,\xi}+ \sum_{\xi \in \OMN}\sqrt{\lambda_2^{\xi}}\left[\overline{\widehat{\w}}^T\widehat{\V}^{2,\xi}\right]\widehat{\V}^{2,\xi}+ \sum_{\xi \in \OMN}\sqrt{\lambda_3^{\xi}}\left[\overline{\widehat{\w}}^T\widehat{\V}^{3,\xi}\right]\widehat{\V}^{3,\xi}.
\end{equation}

Note that the three channels of $\w$ are independent. Furthermore, for $1 \leq j \leq 3$

\begin{align}
	\left(\overline{\widehat{\V}}^{j,\xi}\right)^T\widehat{\Y}
	&  = \sqrt{\lambda_j^{\xi}} \left[\overline{\widehat{\w}}^T\widehat{\V}^{j,\xi}\right]\left\|\widehat{\V}^{j,\xi}\right\|^2  = \sqrt{\lambda_j^{\xi}} \left[\overline{\widehat{\w}}^T\widehat{\V}^{j,\xi}\right]\\
	\left|\left(\overline{\widehat{\V}}^{j,\xi}\right)^T\widehat{\Y}\right|^2
	&  = \lambda_j^{\xi} \left|\overline{\widehat{\w}}^T\widehat{\V}^{j,\xi}\right|^2 \\
	\E\left[\left|\left(\overline{\widehat{\V}}^{j,\xi}\right)^T\widehat{\Y}\right|^2\right]
	& = \lambda_j^{\xi} \E\left[\left|\overline{\widehat{\w}}^T\widehat{\V}^{j,\xi}\right|^2\right] \\
	\E\left[\left|\overline{\widehat{\w}}^T\widehat{\V}^{j,\xi}\right|^2\right]
	& = \sum_{1 \leq c_1,c_2 \leq 3} \E\left[\overline{\widehat{\w}}_{c_1}(\xi)\widehat{\w}_{c_2}(\xi)\right]\widehat{\V}_{c_1}^{j,\xi}(\xi)\overline{\widehat{\V}}_{c_2}(\xi) \text{ by \Cref{eq:v_inner} }\\
	& = \sum_{1 \leq c \leq 3} \E\left[\left|\widehat{\w}_{c}(\xi)\right|^2\right]\left|\widehat{\V}_{c}^{j,\xi}(\xi)\right|^2 \text{ because the channels are independent} \\
	& = MN \sum_{1 \leq c \leq 3} \left|\widehat{\V}_{c}^{j,\xi}(\xi)\right|^2 \text{ because $ \E\left[\left|\widehat{\w}_{c}(\xi)\right|^2\right] = MN$} \\
	& = MN   \text{ by \Cref{eq:v_norm}. }
 \end{align}

Finally,
\begin{align}
	\E\left[\left|\left(\overline{\widehat{\V}}^{j,\xi}\right)^T\widehat{\Y}\right|^2\right] = MN \lambda_1^{\xi} 
\end{align}

Finally, for a given sampling $\left(\Y_k\right)_{1\leq k \leq N_{\text{samples}}}$ following the distribution $\mathcal{N}(\zero,\G)$, an estimator of $\lambda_j^\xi$ is:

\begin{equation}
	\lambda_j^{\xi,\text{emp.}} = \frac{1}{N_{\text{samples}} MN } \sum_{k=1}^{N_{\text{samples}}} \left|\left(\overline{\widehat{\V}}^{j,\xi}\right)^T\widehat{\Y_k}\right|^2.
	\end{equation}

	The empirical Wasserstein distance between the Gaussian distribution $\mathcal{N}(\zero,\Gamma)$ and the ADSN model with texton $\bt$ is:

	\begin{equation}
		\Wemp(\mathcal{N}^{\text{emp.}}(\zero,\G),\ADSN(\U)) = \sqrt{\sum_{\xi \in \OMN} \left(\left(\sqrt{\lambda_1^{\xi,\text{emp.}}} - \sqrt{\lambda_1^{\xi,\ADSN}}\right)^2 + \lambda_2^{\xi,\text{emp.}}+\lambda_3^{\xi,\text{emp.}}\right)}
	\end{equation}
	
	with $\lambda_1^{\xi,\ADSN} = |\widehat{\bt}_1(\xi)|^2+|\widehat{\bt}_2(\xi)|^2+|\widehat{\bt}_3(\xi)|^2$ for $\xi \in \OMN$.

	Furthermore, the computations can be vectorized by componentwise products in the Fourier domain.

\end{document}